\documentclass{article}

\usepackage{iclr2027_conference}

\iclrfinalcopy
\usepackage{times}

\usepackage[utf8]{inputenc}
\usepackage[T1]{fontenc}

\usepackage{microtype}
\usepackage{graphicx}
\usepackage{subcaption}
\usepackage{booktabs}
\usepackage{multirow}
\usepackage{wrapfig}
\usepackage{amsmath,amssymb,amsfonts,amsthm,mathtools,bm}
\usepackage{nicefrac}
\usepackage{textcomp}
\usepackage{url}
\usepackage{xcolor}
\usepackage{hyperref}
\usepackage{enumitem}

\usepackage{wrapfig}
\usepackage{booktabs}
\usepackage{multirow}
\usepackage[table]{xcolor}
\usepackage{graphicx}

\usepackage{algorithm}
\usepackage{algpseudocode}
\newcommand{\clip}{\operatorname{clip}}
\newcommand{\LogitNormal}{\operatorname{LogitNormal}}

\usepackage{mathtools}
\usepackage{microtype}

\newcommand{\norm}[1]{\left\lVert #1\right\rVert}

\usepackage[most]{tcolorbox}
\tcbuselibrary{theorems,skins,breakable}
\definecolor{ash}{HTML}{E8E8E8}
\definecolor{Blue}{rgb}{0,0,1}

\usepackage{amsmath,amsfonts,bm}

\def\eqref#1{equation~\ref{#1}}
\def\1{\bm{1}}

\DeclareMathAlphabet{\mathsfit}{\encodingdefault}{\sfdefault}{m}{sl}
\SetMathAlphabet{\mathsfit}{bold}{\encodingdefault}{\sfdefault}{bx}{n}

\newcommand{\E}{\mathbb{E}}

\hypersetup{
  colorlinks=true,
  linkcolor=blue!60!black,
  citecolor=blue!60!black,
  urlcolor=blue!60!black,
}

\newtheoremstyle{customthm}
  {12pt}{12pt}{\itshape}{0pt}{\bfseries}{.}{.5em}{}
\theoremstyle{customthm}
\newtheorem{theorem}{Theorem}

\newtheorem{remark}[theorem]{Remark}

\newtheoremstyle{customdef}
  {8pt}{8pt}{\normalfont}{0pt}{\bfseries}{.}{.5em}{}
\theoremstyle{customdef}

\newtcolorbox{intuition}{
  colback=blue!3!white,
  colframe=blue!40!black,
  fonttitle=\bfseries,
  title=Intuition,
  breakable,
  left=6pt, right=6pt, top=4pt, bottom=4pt,
  boxrule=0.5pt,
  arc=2pt,
}

\newtcolorbox{proofsketchbox}{
  colback=gray!4!white,
  colframe=gray!50!black,
  fonttitle=\bfseries,
  title=Proof Sketch,
  breakable,
  left=6pt, right=6pt, top=4pt, bottom=4pt,
  boxrule=0.4pt,
  arc=2pt,
}

\providecommand{\1}{\mathbf{1}}

\newcommand{\sg}{\operatorname{sg}}

\title{SALD: Self-Referenced Advantage Learning for Diffusion Models}

\author{%
\textbf{Aryan Das}$^{1}$ \quad
\textbf{Surjo Dey}$^{2}$ \quad
\textbf{Koushik Biswas}$^{3}$ \quad
\textbf{Swalpa Kumar Roy}$^{4}$
\\[3pt]
\textbf{Moloud Abdar}$^{5}$ \quad
\textbf{Arnab Bhattacharya}$^{1}$ \quad
\textbf{Vinay Kumar Verma}$^{1}$
\\[7pt]
\small
$^{1}$Indian Institute of Technology Kanpur, India\\
$^{2}$Rajiv Gandhi Institute of Petroleum Technology, India\\
$^{3}$IIIT Delhi, India\\
$^{4}$Tezpur University, India\\
$^{5}$The University of Queensland, Australia
}

\begin{document}

\maketitle

\vspace*{-12pt}

% Sections
%%%%%%%%%%%%%%%%%%%%%%%%%%%%%%%%%%%%%%%%%%%%%%%%%%%%%%%%%%%%%%%%
\begin{abstract}
Recent work on language-model adaptation has shown that single models can obtain informative training signals by evaluating their behavior in demonstration- or feedback-augmented contexts, with the help of a teacher network, which is driven by the students' learned parameters. Inspired by this internal-reference principle, we investigate how diffusion models can identify self-referenced training signals without external demonstrations or teacher networks. We introduce SALD, a self-referenced training framework that evaluates each image-caption pair at two noise levels using the same model. The easier, lower-noise path is evaluated without gradient tracking to provide a reference, while the harder, higher-noise path provides the training gradient. Rather than directly distilling the easy-path prediction, SALD uses the difference between two path errors to adapt the hard-path objective. The proposed Advantage-Guided Diffusion (AGD) converts this relative error into a differentiable sample-level weight. Temporal Advantage Memory (TAM) accumulates relative difficulty across training and adapts the future gap between the two noise levels. Spectral Advantage Decomposition (SAD) further compares the residual power spectra of the two paths and constructs a differentiable, frequency-derived latent-element weight. All components share a single set of model parameters, requiring neither an external teacher network nor additional trainable parameters during training or inference, and no modification to the inference procedure. Experiments across multiple architectures and datasets demonstrate consistent improvements in generation quality, while component-wise ablations quantify the contributions of the proposed components.

\end{abstract}
%%%%%%%%%%%%%%%%%%%%%%%%%%%%%%%%%%%%%%%%%%%%%%%%%%%%%%%%%%%%%%%%
%%%%%%%%%%%%%%%%%%%%%%%%%%%%%%%%%%%%%%%%%%%%%%%%%%%%%%%%%%%%%%%%
\begin{figure}[H]
    \vspace{-6mm}
    \centering
    \includegraphics[width=0.82\linewidth]{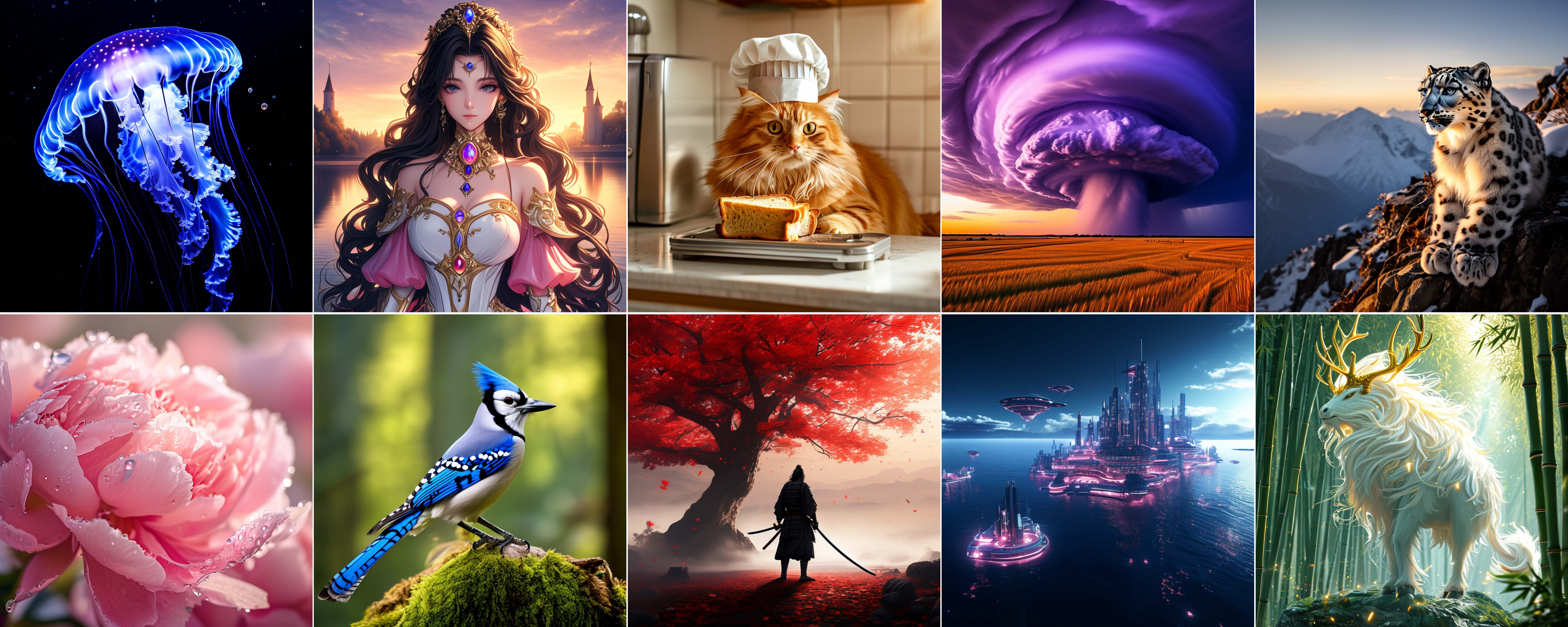} % 0.82
    \vspace{-2mm}
    \caption{Images generated by SALD using the SANA 0.6B model at a resolution of $1024 \times 1024$.}
    \label{fig:abstract}
    \vspace{-4mm}
\end{figure}

\section{Introduction}
\label{sec:intro}
\vspace{-3mm}

Diffusion \cite{ho2020denoising,song2021scorebased} and flow-matching
\cite{lipman2023flow,albergo2023stochastic} based generative models have become
standard backbones for high-fidelity image synthesis, with Diffusion
Transformers further improving scalability and generation quality
\cite{peebles2023scalable}.  In parallel, recent work on language-model
adaptation has shown that a model can obtain richer online training signals by
evaluating its behavior under additional information.  Self-Distillation
Fine-Tuning conditions a self-teacher on expert demonstrations, whereas
Self-Distillation Policy Optimization conditions a self-teacher on feedback
obtained from the model's own attempts
\cite{shenfeld2026selfdistillation,hubotter2026reinforcement}.  These approaches
show that useful supervision can be constructed without relying on a stronger
external teacher.  Transferring this principle to diffusion training is not
straightforward: diffusion models do not naturally receive textual
demonstrations or retrospective feedback, and their training states are sampled
from a prescribed corruption process rather than from autoregressive policy
rollouts.  Moreover, standard diffusion objectives still average gradients
uniformly across training pairs and latent elements, although their difficulty
varies substantially across images, captions, noise levels, and visual details
\cite{vincent2011score,choi2022perception}.  Neural networks also exhibit
spectral bias, learning coarse low-frequency structure before fine
high-frequency detail \cite{rahaman2019spectral,wang2025spectral}.  Consequently,
uniform training can continue to emphasize already learned signals while
underweighting examples and visual components that remain difficult.
\vspace{-2mm}
\begin{tcolorbox}[
enhanced,
colback=yellow!8!white,
colframe=black,
boxrule=0.7pt,
arc=2mm,
width=\textwidth,
top=1.5mm,
bottom=1.5mm,
left=3mm,
right=3mm,
% before skip=6pt,
% after skip=6pt
]
\centering
\itshape\bfseries
Can a diffusion model construct its own easier reference and use it to adapt training across harder examples, prompts, noise levels, and visual frequencies,
without teachers, demonstrations, or feedback datasets?
\end{tcolorbox}

A unified answer remains missing because existing methods address only parts of
this problem.  Noise-level weighting
\cite{karras2022elucidating,esser2024scaling,hang2024noise}
adjusts training according to timestep or signal-to-noise ratio, but normally
applies the same rule to all examples evaluated at a given noise level. Noise
selection methods show that different noise instances provide unequal training
signals \cite{chen2024noises}, but they do not maintain a persistent estimate
of how the difficulty of an individual image-caption pair changes throughout
training.  Curriculum learning \cite{bengio2009curriculum} orders examples by
difficulty, but typically depends on a fixed or externally defined curriculum
and are not adaptive noise selection. Distributionally robust optimization emphasizes difficult subpopulations
\cite{sagawa2020dro}, whereas frequency-aware objectives address spectral
imbalance; neither directly combines the sample-level, temporal, and spectral
difficulty in diffusion training.  Likewise, demonstration- or
feedback-conditioned language-model self-distillation cannot be transferred
directly because diffusion training provides neither an in-context
demonstration, nor a feedback-conditioned output distribution to distill
\cite{shenfeld2026selfdistillation,hubotter2026reinforcement}.  These limitations
motivate a self-referenced diffusion objective in which an easier view of the
same training pair provides a relative difficulty signal rather than a teacher
prediction target.

We introduce SALD, which evaluates every
image-caption pair at two noise levels using one shared model: an easier
low-noise path and a harder high-noise path.  The easy path is evaluated without
gradient tracking and serves as an internal reference, while model parameters
are updated through the hard path.  SALD does not directly distill the easy-path
prediction into the hard path.  Instead, it measures the increase in weighted
prediction error from the easy path to the hard path and uses this relative
signal to adapt the hard-path objective.  \emph{Advantage-Guided Diffusion}
(AGD) converts the scalar error difference into a normalized, differentiable
sample weight, emphasizing training pairs whose error grows under the harder
condition.  \emph{Temporal Advantage Memory} (TAM) maintains an exponential
moving average of relative difficulty for every training pair and uses it to
adapt future noise-level gaps, forming a continuous per-example curriculum
\cite{bengio2009curriculum}.  \emph{Spectral Advantage Decomposition} (SAD)
compares the residual power spectra of the two paths and constructs a
differentiable, frequency-derived latent-element weight, providing a finer
training signal related to spectral bias
\cite{rahaman2019spectral,wang2025spectral}.  All components reuse the same set
of model parameters and require neither a teacher network nor
additional trainable model parameters during training or inference.  The
method introduces no change to the inference procedure.  Our contributions are
summarized as follows:
\vspace{-2mm}
\begin{itemize}[leftmargin=*]
    \item We introduce AGD, a self-referenced weighting mechanism that estimates
    sample importance from the relative error between hard and easy noise paths,
    without any teachers, demonstration datasets, or architectural changes.

    \item We introduce TAM and SAD to extend relative difficulty across training
    history and spectral structure, enabling an adaptive per-example noise
    curriculum and frequency-derived, latent-element weighting.

    \item We characterize the optimization behavior of the resulting
    differentiable weighting objective and evaluate SALD across multiple
    architectures and datasets.  The experiments demonstrate consistent gains
    in generation quality, while component-wise ablations quantify the
    contributions of AGD, TAM, and SAD.
\end{itemize}

% \vspace{-7mm}
\section{Proposed Approach}
% \vspace{-2mm}
\label{sec:problem}
The same diffusion model evaluates each image-caption pair at two noise levels. The lower-noise evaluation forms an \emph{easy path}, which is computed without gradient tracking and serves as an internal reference. The higher-noise evaluation forms a \emph{hard path}, through which the model parameters are updated. By comparing these paths, SALD addresses three questions: \emph{(i) Which examples are currently difficult? (ii) How large should the easy-to-hard noise gap be for each example? and (iii) At which frequencies does the hard-path residual remain large?} Advantage-Guided Diffusion (AGD), Temporal Advantage Memory (TAM), and Spectral Advantage Decomposition (SAD) address these questions, respectively. All components share the same model parameters; therefore, SALD requires neither an external teacher nor an additional trainable model during training or inference.

%%%%%%%%%%%%%%%%%%%%%%%%%%%%%%%%%%%%
\begin{figure*}[t]
    % \vspace{-5mm}
    \centering
    \includegraphics[width=0.99\textwidth]{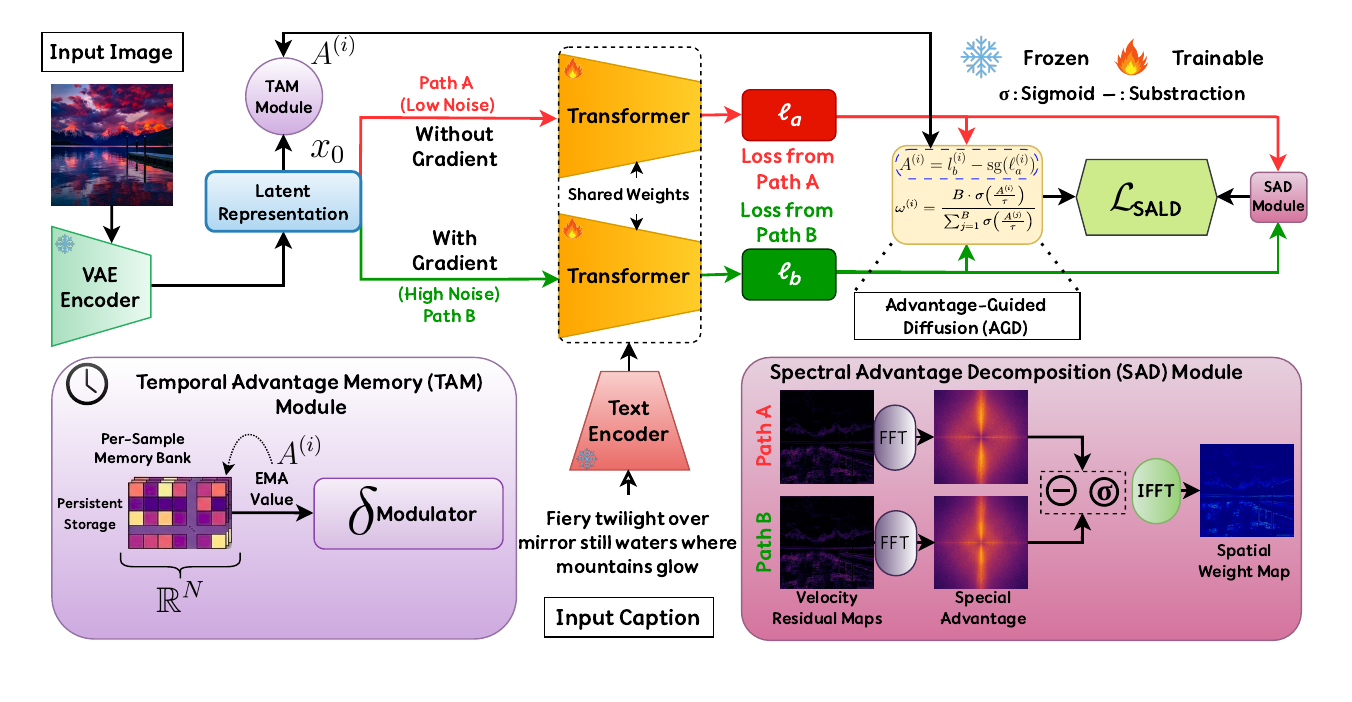} % height= 6cm, 
    % \vspace{-6mm}
    \caption{\emph{SALD framework.} A frozen VAE encodes the input latent, processed through two paths: Low-Noise Path~A (w/o gradient) and High-Noise Path~B (with gradient). AGD computes an advantage for sample reweighting, while TAM stabilizes it via per-sample EMA, forming a persistent curriculum across timestep gaps. SAD compares FFT power spectra of velocity residual maps, applies a sigmoid gate, and reconstructs a spatial weight map via IFFT to focus gradients on high-frequency error regions. Together, these define $\mathcal{L}_{\text{SALD}}$ for adaptive gradient allocation across samples.}
    % }
    \label{fig:model}
    \vspace{-4mm}
\end{figure*}
%%%%%%%%%%%%%%%%%%%%%%%%%%%%%%%%%%%%

\vspace{-2mm}
%%%%%%%%%%%%%%%%%%%%%%%%%%%%%%
\subsection{Flow-Matching Setup}
\label{sec:flow-setup}

Let $x_0\in\mathbb{R}^{C\times H\times W}$ be the latent representation of an
image, let $c$ be its text condition, and let
$\epsilon\sim\mathcal{N}(0,I)$ be Gaussian noise.  At noise level
$t\in[0,1]$, the model input is
\begin{equation}
    z_t=(1-t)x_0+t\epsilon.
    \label{eq:flow-path}
\end{equation}
The desired velocity is simply $u=\epsilon-x_0$.
% \begin{equation}
%     u=\epsilon-x_0.
%     \label{eq:velocity-target}
% \end{equation}
The model $v_\theta(z_t,t,c)$ predicts this velocity.  To keep the remaining
notation short, define the residual (velocity prediction error) and its scalar loss by
\begin{align}
    r(t,\epsilon)
    &=v_\theta(z_t,t,c)-(\epsilon-x_0),
    \label{eq:residual-general}
    \\
    \ell(t,\epsilon)
    &=\frac{w(t)}{D}\left\|r(t,\epsilon)\right\|_2^2,
    \qquad D=CHW,
    \label{eq:loss-general}
\end{align}
where $w(t)$ is the usual time-dependent loss weight. SALD now chooses two noise levels for the same training samples: $t_{\mathrm e}<t_{\mathrm h},$
% \begin{equation}
%     t_{\mathrm e}<t_{\mathrm h},
%     \label{eq:easy-hard-times}
% \end{equation}
where $t_{\mathrm e}$ is the easy level and $t_{\mathrm h}$ is the hard level.
The two paths use independent noise samples $\epsilon_{\mathrm e}$ and
$\epsilon_{\mathrm h}$.  We write
\begin{align}
    r_{\mathrm e}&=r(t_{\mathrm e},\epsilon_{\mathrm e}),
    &\ell_{\mathrm e}&=\ell(t_{\mathrm e},\epsilon_{\mathrm e}),
    \\
    r_{\mathrm h}&=r(t_{\mathrm h},\epsilon_{\mathrm h}),
    &\ell_{\mathrm h}&=\ell(t_{\mathrm h},\epsilon_{\mathrm h}).
    \label{eq:easy-hard-quantities}
\end{align}
The easy quantities are computed without gradient tracking, whereas the hard quantities remain differentiable. We choose the two noise levels such that: $t_{\mathrm e}^{(i)} < t_{\mathrm h}^{(i)}$, where \(t_{\mathrm e}^{(i)}\) and \(t_{\mathrm h}^{(i)}\) denote the easy and hard noise levels for training sample \(i\), respectively. Their separation is controlled by a noise-level gap \(\delta_i\) (defined in Eq. \ref{eq:tam-gap}). A small \(\delta_i\) produces two similar noise states, whereas a large \(\delta_i\) creates a more difficult hard state. Temporal Advantage Memory (TAM) later adapts \(\delta_i\) according to the historical difficulty of each training pair. In implementation, \(\delta_i\) corresponds to the number of steps separating the two levels in the discretized noise schedule.

\vspace{-2mm}
 %%%%%%%%%%%%%%%%%%%%%%%%%%%%%%%%%%%%
\subsection{Advantage-Guided Diffusion (AGD)}
\label{sec:agd}
\vspace{-2mm}
For each sample $i$, AGD measures how much the loss increases from the easy path to the hard path:
\begin{equation}
    A_i
    =\ell_{\mathrm h,i}
    -\sg\!\left(\ell_{\mathrm e,i}\right).
    \label{eq:advantage}
\end{equation}
Here, $\sg(\cdot)$ means stop-gradient.  Importantly, only the easy reference
is stopped. \emph{This is an approximation to obtain the demonstration or feedback sample, which is difficult to obtain for self-distillation ~\cite{shenfeld2026selfdistillation,hubotter2026reinforcement} in image generations}.  The advantage itself is \emph{not} detached, because it backpropagates the gradient for the hard path as: $ \nabla_\theta A_i=\nabla_\theta\ell_{\mathrm h,i}$
% \begin{equation}
%     \nabla_\theta A_i=\nabla_\theta\ell_{\mathrm h,i}.
%     \label{eq:advantage-gradient}
% \end{equation}

A large positive $A_i$ means that pair $i$ becomes considerably more difficult
at the higher noise level. Further, $A_i$ is normalized into a positive mini-batch weight.  For a batch of size
$B$, define
\begin{equation}
    q_i=\operatorname{sigmoid}(A_i/\tau),
    \qquad
    \omega_i=\frac{Bq_i}{\sum_{j=1}^{B}q_j},
    \label{eq:agd-weight}
\end{equation}
where $\tau>0$ is a temperature.  These weights satisfy
\begin{equation}
    \frac{1}{B}\sum_{i=1}^{B}\omega_i=1.
    \label{eq:weight-normalization}
\end{equation}
Thus, AGD changes the relative importance of examples without altering the average batch weight. Since $A_i$ is differentiable, $\omega_i$ is also differentiable through the hard path. The denominator in Eq.~\eqref{eq:agd-weight} further couples the weights of all examples in the batch. Harder samples therefore receive larger weights, and optimizing them encourages their losses to move closer to those of easier samples, enabling the model to generate outputs that progressively reduce the loss on difficult examples.

%%%%%%%%%%%%%%%%%%%%%%%%%%%%%%%%%%%%%%%%%%%%%%%%%%%%%%%%%%%%%%%%%%%%%
\vspace{-2mm}
\subsection{Temporal Advantage Memory (TAM)}
\label{sec:tam}

A single advantage value can be noisy because new Gaussian noise is sampled at
every iteration.  TAM therefore stores one running difficulty value $m_i$ for
each training sample.  The memory vector $m\in\mathbb{R}^{N}$ is initialized to
zero and is a non-trainable buffer.

Let $\Delta$ be the largest allowed gap $s$ be the current training
step, and $K>1$ be the number of curriculum steps.  A base curriculum that
starts at one and gradually reaches $\Delta$ is
\begin{equation}
    \delta_{\mathrm{base}}(s)=
    \begin{cases}
        1+\left\lfloor
        \dfrac{(\Delta-1)s}{K-1}
        \right\rfloor,
        &0\leq s<K,\\[7pt]
        d,\quad d\sim
        \operatorname{Uniform}\{1,\ldots,\Delta_{\max}\},
        &s\geq K.
    \end{cases}
    \label{eq:base-curriculum}
\end{equation}
Before the current pair is evaluated, its stored memory selects the actual gap:
\begin{equation}
    \delta_i=\operatorname{clip}\!\left(
        \operatorname{round}\!\left[
            \delta_{\mathrm{base}}(s)
            \bigl(1+\lambda\tanh(m_i)\bigr)
        \right],
        1,\Delta_{\max}
    \right),
    \label{eq:tam-gap}
\end{equation}
where $\lambda\geq0$ controls the strength of memory modulation. After the current advantage has been computed, TAM updates the memory:
\begin{equation}
    m_i\leftarrow\alpha m_i+(1-\alpha)A_i,
    \qquad 0<\alpha<1.
    \label{eq:tam-update}
\end{equation}
This assignment stores the numerical value of $A_i$; it does not create a
gradient path through the memory buffer.  The updated value is used the next
time the same pair is visited.  Persistent positive memory increases future
noise gaps, while negative memory decreases them.  The $\tanh$ and clipping
operations keep every gap within the valid range $(1-\lambda, 1+\lambda)$.

%%%%%%%%%%%%%%%%%%%%%%%%%%%%%%%%%%%%%%%%%%%%%%%%%%%%%%%%%%%%%%%%%%%%%%%
\vspace{-2mm}
\subsection{Spectral Advantage Decomposition (SAD)}
\label{sec:sad}
\vspace{-2mm}
The proposed AGD assigns a single scalar weight to an entire example at each frequency.  SAD adds frequency
information by comparing the residual maps $r_{\mathrm e,i}$ and
$r_{\mathrm h,i}$ in the frequency domain for the easy and hard respectively. Let $\mathcal{F}_r$ be an orthonormal two-dimensional real
FFT applied independently to every latent channel.  The easy and hard residual
power spectra are
\begin{equation}
    P_{\mathrm e,i}
    =\left|\mathcal{F}_r(r_{\mathrm e,i})\right|^{\odot2},
    \qquad
    P_{\mathrm h,i}
    =\left|\mathcal{F}_r(r_{\mathrm h,i})\right|^{\odot2}.
    \label{eq:power-spectra}
\end{equation}
Their spectral advantage is
\begin{equation}
    S_i=P_{\mathrm h,i}-\sg\!\left(P_{\mathrm e,i}\right).
    \label{eq:spectral-advantage}
\end{equation}
Again, only the easy reference is stopped.  Therefore, $S_i$ remains
differentiable through the hard residual.  A positive value of $S_i$ at a
frequency means that the hard path contains more residual energy than the easy
reference at that frequency.

The spectral advantage is mapped to a bounded gate; the gate is then back-projected to the latent domain and normalized:
\begin{align}
    G_i & =\operatorname{sigmoid}(S_i/\tau_s) \quad \& \quad
    \widetilde{M}_i
    =\left|\mathcal{F}_r^{-1}(G_i)\right|, \quad \tau_s>0
    \label{eq:unnormalized-mask}
    \\
    M_i
    &=\frac{D\,\widetilde{M}_i}
    {\sum_{p=1}^{D}\widetilde{M}_{i,p}+\varepsilon_m},
    \qquad \varepsilon_m=10^{-8}.
    \label{eq:normalized-mask}
\end{align}
The normalization keeps the mean weight approximately equal to one.
Crucially, there is no stop-gradient in Eq.~\ref{eq:normalized-mask}; $M_i$
is differentiable through $P_{\mathrm h,i}$.

\(M_i\) is a latent-domain weight obtained from the frequency-wise comparison of the easy and hard residuals. Because SAD uses residual power spectra, it discards the Fourier phase information that represents spatial location. Therefore, \(M_i\) indicates which frequency patterns remain difficult, but it does not identify where the corresponding errors occur in the latent representation. We consequently describe \(M_i\) as a frequency-derived latent weight rather than a localized pixel-level advantage map.

\vspace{-3mm}
\subsection{Unified SALD Objective}
\label{sec:unified}
For pair $i$, SAD first weights the squared hard-path residual element by
element: \vspace{-2mm}
\begin{equation}
    \ell_i^{\mathrm{SAD}}
    =\frac{w(t_{\mathrm h,i})}{D}
    \sum_{p=1}^{D}M_{i,p}\,r_{h,i,p}^{2}
    \label{eq:sad-loss}
\end{equation}
AGD then weights the examples in the mini-batch: \vspace{-2mm}
\begin{equation}
    \mathcal{L}_{\mathrm{SALD}}(\theta)
    =\frac{1}{B}\sum_{i=1}^{B}
    \omega_i\ell_i^{\mathrm{SAD}}
    \label{eq:sald-objective}
    \vspace{-2mm}
\end{equation}

The gradient follows the displayed equations exactly.  Since neither
$\omega_i$ nor $M_i$ is detached and $\nabla_\theta\ell_i^{\mathrm{SAD}}$ includes the gradient through $M_i$.
The only stopped branches are the easy-path references.  TAM is stored as
a non-trainable state and is not backpropagated through time.
SALD uses two forward evaluations, but only the hard evaluation stores a
backward graph.  It introduces an $N$-value TAM buffer and FFT operations, but
no additional trainable or non-trainable (teacher) model parameters are required during training or inference, and the inference procedure is unchanged.

Finally, Eq.~\ref{eq:sald-objective} contains no term that directly matches
the hard prediction to the easy prediction.  The easy path supplies a
reference for constructing adaptive weights; it is not a teacher output target.
For this reason, the method is more precisely described as
\emph{self-referenced advantage learning} than as teacher-student
self-distillation.

\vspace{-2mm}
\subsection{Theoretical Interpretation}
\vspace{-2mm}
\label{sec:theory}

\begin{theorem}[Conditional convergence of the differentiable objective]
Assume that $F_m$ is lower bounded by $F_m^\star$ and is $L$-smooth.  Let
$g_s=\nabla_\theta
\widehat{\mathcal{L}}_{\mathrm{SALD}}(\theta_s;\xi_s,m)$ include all gradient
paths through $\omega_i$ and $M_i$.  Assume
\begin{equation}
    \E[g_s\mid\theta_s]=\nabla F_m(\theta_s),
    \qquad
    \E\!\left[
        \left\|g_s-\nabla F_m(\theta_s)\right\|_2^2
    \right]\leq\nu^2.
    \label{eq:gradient-assumptions}
\end{equation}
For the update $\theta_{s+1}=\theta_s-\eta g_s$ with
$0<\eta\leq1/L$,
\begin{equation}
    \frac{1}{J}\sum_{s=0}^{J-1}
    \E\left\|\nabla F_m(\theta_s)\right\|_2^2
    \leq
    \frac{2\bigl(F_m(\theta_0)-F_m^\star\bigr)}{\eta J}
    +L\eta\nu^2.
    \label{eq:convergence-bound}
\end{equation}
\vspace{-2mm}
\end{theorem}
\vspace{-4mm}
This standard result guarantees convergence toward a stationary point only
under the stated assumptions.  
%It does not prove that SALD converges faster than uniform flow matching, and it does not imply that the two objectives have the same stationary points. 
Exact statement supported by the spectral gate: for every frequency coordinate $u$,
\begin{equation}
    G_i(u)>\frac{1}{2}
    \quad\Longleftrightarrow\quad
    P_{\mathrm h,i}(u)>
    \sg\!\left(P_{\mathrm e,i}(u)\right).
    \label{eq:gate-property}
\end{equation}
Thus, the gate is larger where the hard residual has more spectral energy than
the easy reference.  However, the final loss multiplies by $M_i$ in the latent
domain rather than multiplying Fourier coefficients directly by $G_i$. Table~\ref{tab:eval_sana_datasets} in the Supplementary empirically support the theoretical claim and shows the faster and better convergence.
\vspace{-3mm}
%%%%%%%%%%%%%%%%%%%%%%%%%%%%%%%%%%
\section{Experiments}
\label{sec:experiments}
\vspace{-2mm}
%%%%%%%%%%%%%%%%%%%%%%%%%%%%%%%%%%

%%%%%%%%%%%%%%%%%%%
\begin{table*}[t]
\vspace{-4mm}
\centering
\small
\caption{Quantitative results on ($512\times512$) resolution datasets. Best values per metric are in \textcolor{blue}{Blue}.}
\vspace{-1mm}
\label{tab:results_large}
\setlength{\tabcolsep}{1pt}
\renewcommand{\arraystretch}{1.1}
\resizebox{\textwidth}{!}{%
\begin{tabular}{l|l|ccc|ccc|ccc|ccc|ccc}
\toprule
\multirow{2}{*}{\textbf{Model}} & \multirow{2}{*}{\textbf{Method}} 
  & \multicolumn{3}{c|}{\textbf{COCO (2017)}} 
  & \multicolumn{3}{c|}{\textbf{MultiGen-20M}} 
  & \multicolumn{3}{c|}{\textbf{Flickr8k}}
  & \multicolumn{3}{c|}{\textbf{CUB-200}} 
  & \multicolumn{3}{c}{\textbf{Oxford 102 Flowers}} \\
\cmidrule(lr){3-5}
\cmidrule(lr){6-8}
\cmidrule(lr){9-11}
\cmidrule(lr){12-14}
\cmidrule(lr){15-17}
 & 
  & \textbf{FID}$\downarrow$ & \textbf{CLIP}$\uparrow$ & \textbf{LPIPS}$\downarrow$ 
  & \textbf{FID}$\downarrow$ & \textbf{CLIP}$\uparrow$ & \textbf{LPIPS}$\downarrow$ 
  & \textbf{FID}$\downarrow$ & \textbf{CLIP}$\uparrow$ & \textbf{LPIPS}$\downarrow$
  & \textbf{FID}$\downarrow$ & \textbf{CLIP}$\uparrow$ & \textbf{LPIPS}$\downarrow$
  & \textbf{FID}$\downarrow$ & \textbf{CLIP}$\uparrow$ & \textbf{LPIPS}$\downarrow$ \\
\midrule

\multirow{5}{*}{\textbf{SANA (1.6B)}} & Base 
  &26.49 & 26.44 & 0.7411
  &26.44 & 26.25 & 0.7541
  &27.22 & 28.49 & 0.7336 
  &21.63 & 26.01 & 0.7748 
  &33.11 & 26.35 & 0.7896 \\
 & SpeeD 
  &25.38 & 26.83 & 0.7409
  &26.11 & 26.26 & 0.7558
  &26.87 & 28.52 & 0.7324 
  &21.53 & 26.28 & 0.7753 
  &32.04 & 26.51 & 0.7782 \\
 & Temporal Diff 
  &26.36 & 26.87 & 0.7405
  &26.31 & 26.29 & 0.7539
  &26.85 & 28.85 & 0.7325 
  &20.72 & 26.25 & 0.7766 
  &33.10 & 26.54 & 0.7801 \\
 & SRA 
  &26.89 & 26.32 & 0.7461
  &26.05 & 25.89 & 0.7554
  &26.87 & 28.78 & 0.7341 
  &20.78 & 26.18 & 0.7722 
  &32.75 & 26.40 & \textcolor{blue}{\textbf{0.7781}} \\
 & \textbf{SALD} 
  &\textcolor{blue}{\textbf{22.76}} & \textcolor{blue}{\textbf{26.88}} & \textcolor{blue}{\textbf{0.7391}}
  &\textcolor{blue}{\textbf{23.86}} & \textcolor{blue}{\textbf{26.29}} & \textcolor{blue}{\textbf{0.7533}}
  &\textcolor{blue}{\textbf{26.75}} & \textcolor{blue}{\textbf{28.88}} & \textcolor{blue}{\textbf{0.7316}} 
  &\textcolor{blue}{\textbf{18.16}} & \textcolor{blue}{\textbf{26.29}} & \textcolor{blue}{\textbf{0.7537}} 
  &\textcolor{blue}{\textbf{27.37}} & \textcolor{blue}{\textbf{26.57}} & \textcolor{blue}{\textbf{0.7781}} \\
\midrule

\multirow{5}{*}{\textbf{SANA (0.6B)}} & Base 
  &31.59 & 26.80 & 0.7408
  &31.93 & 26.68 & 0.7495
  &28.83 & 28.53 & 0.7351
  &26.06 & 26.40 & 0.7807 
  &30.64 & 26.43 & 0.7805 \\
 & SpeeD 
  &31.47 & \textcolor{blue}{\textbf{27.06}} & 0.7411
  &31.57 & 26.69 & 0.7485
  &28.42 & 28.49 & 0.7313 
  &25.98 & \textcolor{blue}{\textbf{26.61}} & 0.7611 
  &32.04 & 26.54 & 0.7753 \\
 & Temporal Diff 
  &31.28 & 27.03 & 0.7424
  &31.67 & 26.66 & 0.7498
  &28.66 & 28.53 & 0.7330 
  &26.18 & 26.53 & 0.7775 
  &33.10 & 26.61 & 0.7736 \\
 & SRA 
  &31.26 & 26.92 & 0.7397
  &31.63 & 26.56 & 0.7473
  &28.86 & 28.21 & 0.7314 
  &24.94 & 26.56 & 0.7730 
  &29.72 & 26.54 & 0.7748 \\
 & \textbf{SALD} 
  &\textcolor{blue}{\textbf{29.31}} & \textcolor{blue}{\textbf{27.06}} & \textcolor{blue}{\textbf{0.7374}}
  &\textcolor{blue}{\textbf{29.15}} & \textcolor{blue}{\textbf{26.69}} & \textcolor{blue}{\textbf{0.7472}}
  &\textcolor{blue}{\textbf{27.91}} & \textcolor{blue}{\textbf{28.59}} & \textcolor{blue}{\textbf{0.7275}}  
  &\textcolor{blue}{\textbf{23.70}} & \textcolor{blue}{\textbf{26.61}} & \textcolor{blue}{\textbf{0.7686}}
  &\textcolor{blue}{\textbf{25.83}} & \textcolor{blue}{\textbf{26.65}} & \textcolor{blue}{\textbf{0.7728}} \\
\midrule

\multirow{5}{*}{\textbf{PixArt-$\Sigma$}} & Base 
  &29.26 & 25.82 & 0.7653
  &31.62 & 25.83 & 0.7306
  &30.83 & 28.44 & 0.7175 
  &22.24 & 25.48 & 0.7491 
  &37.88 & 25.37 & 0.7687 \\
 & SpeeD 
  &29.25 & 25.80 & 0.7533
  &31.42 & 25.83 & \textcolor{blue}{\textbf{0.7292}}
  &30.76 & 28.44 & 0.7153 
  &22.15 & 25.42 & 0.7459 
  &37.16 & 25.39 & 0.7699 \\
 & Temporal Diff 
  &29.21 & 25.85 & 0.7594
  &31.31 & 25.41 & 0.7306
  &30.71 & 28.36 & 0.7170 
  &22.14 & 25.40 & \textcolor{blue}{\textbf{0.7452}} 
  &37.49 & 25.39 & 0.7681 \\
 & SRA 
  &29.21 & 25.88 & 0.7575
  &31.28 & 25.88 & 0.7297
  &30.72 & 28.42 & 0.7159
  &22.15 & 25.00 & 0.7566 
  &37.69 & 25.38 & 0.7706 \\
 & \textbf{SALD} 
  &\textcolor{blue}{\textbf{28.14}} & \textcolor{blue}{\textbf{26.00}} & \textcolor{blue}{\textbf{0.7516}}
  &\textcolor{blue}{\textbf{30.47}} & \textcolor{blue}{\textbf{25.93}} & \textcolor{blue}{\textbf{0.7292}}
  &\textcolor{blue}{\textbf{30.68}} & \textcolor{blue}{\textbf{28.45}} & \textcolor{blue}{\textbf{0.7141}} 
  &\textcolor{blue}{\textbf{21.71}} & \textcolor{blue}{\textbf{25.48}} & \textcolor{blue}{\textbf{0.7452}} 
  &\textcolor{blue}{\textbf{36.48}} & \textcolor{blue}{\textbf{25.42}} & \textcolor{blue}{\textbf{0.7682}} \\
\bottomrule
\end{tabular}%
}
\vspace{-5mm}
\end{table*}
\begin{wraptable}{r}{0.65\linewidth}
\vspace{-4mm}
\centering
\setlength{\tabcolsep}{4pt}
\Large
\caption{Quantitative results on $1024\!\times\!1024$ resolution images.
Best values per metric within each model are in \textcolor{blue}{Blue}.}
\label{tab:results1024}
\vspace{-1mm}
\resizebox{\linewidth}{!}{
\begin{tabular}{l|l|ccc|ccc}
\toprule
\multirow{2}{*}{\textbf{Model}} & \multirow{2}{*}{\textbf{Method}} 
  & \multicolumn{3}{c|}{\textbf{MAGICK}}
  & \multicolumn{3}{c}{\textbf{DALL$\cdot$E~3}} \\
\cmidrule(lr){3-5}\cmidrule(lr){6-8}
 & 
  & \textbf{FID}$\downarrow$ & \textbf{CLIP}$\uparrow$ & \textbf{LPIPS}$\downarrow$ 
  & \textbf{FID}$\downarrow$ & \textbf{CLIP}$\uparrow$ & \textbf{LPIPS}$\downarrow$ \\
\midrule

\multirow{5}{*}{\textbf{SANA 0.6B}} & Base
  & 38.59 & 26.68 & 0.7470 
  & 14.68 & 29.96 & 0.7882 \\
 & SpeeD 
  & 38.27 & 26.65 & 0.7303 
  & 14.18 & 30.05 & 0.7878 \\
 & Temporal Diff 
  & 37.95 & 26.69 & 0.7294 
  & 13.84 & 30.38 & 0.7854 \\
 & SRA 
  & 38.15 & 26.70 & 0.7297 
  & 13.96 & 30.73 & 0.7839 \\
 & \textbf{SALD}
  & \textcolor{blue}{\textbf{36.74}} 
  & \textcolor{blue}{\textbf{26.85}} 
  & \textcolor{blue}{\textbf{0.7230}}
  & \textcolor{blue}{\textbf{11.98}} 
  & \textcolor{blue}{\textbf{31.06}} 
  & \textcolor{blue}{\textbf{0.7826}} \\
\midrule

\multirow{5}{*}{\textbf{SANA 1.6B}} & Base
  & 39.46 & 27.11 & 0.7947 
  & 16.19 & 30.95 & 0.8237 \\
 & SpeeD 
  & 39.34 & 27.28 & 0.7920 
  & 16.05 & 30.91 & \textcolor{blue}{\textbf{0.8206}} \\
 & Temporal Diff 
  & 39.00 & 27.26 & 0.7929 
  & 16.04 & 30.97 & 0.8283 \\
 & SRA 
  & 39.26 & 27.11 & 0.7879 
  & 16.15 & 30.94 & 0.8246 \\
 & \textbf{SALD}
  & \textcolor{blue}{\textbf{37.39}} 
  & \textcolor{blue}{\textbf{27.54}} 
  & \textcolor{blue}{\textbf{0.7877}} 
  & \textcolor{blue}{\textbf{15.82}} 
  & \textcolor{blue}{\textbf{30.99}} 
  & \textcolor{blue}{\textbf{0.8206}} \\
\midrule

\multirow{5}{*}{\textbf{PixArt-$\Sigma$}} & Base
  & 36.72 & 26.31 & 0.7311 
  & 12.03 & 30.23 & 0.8056 \\
 & SpeeD 
  & 36.56 & \textcolor{blue}{\textbf{26.47}} & 0.7197 
  & 11.76 & 30.29 & 0.7844 \\
 & Temporal Diff 
  & 36.49 & 26.28 & 0.7295 
  & 11.65 & 30.66 & 0.7966 \\
 & SRA 
  & 36.52 & 26.44 & 0.7399 
  & 11.73 & 31.29 & 0.7864 \\
 & \textbf{SALD}
  & \textcolor{blue}{\textbf{35.38}} 
  & \textcolor{blue}{\textbf{26.47}} 
  & \textcolor{blue}{\textbf{0.7191}}
  & \textcolor{blue}{\textbf{11.06}} 
  & \textcolor{blue}{\textbf{31.35}} 
  & \textcolor{blue}{\textbf{0.7700}}  \\

\midrule

\multirow{5}{*}{\textbf{Flux.2 Klein}} & Base
  & 35.93 & 24.93 & 0.7804 
  & 14.11 & 31.79 & 0.7666 \\
 & SpeeD 
  & 35.22 & 24.82 & 0.7904 
  & 14.00 & 31.79 & 0.7659 \\
 & Temporal Diff 
  & 35.92 & 24.93 & 0.7803 
  & 13.62 & 31.82 & 0.7618 \\
 & SRA 
  & 35.05 & 24.72 & 0.7914 
  & 13.44 & 31.76 & 0.7622 \\
 & \textbf{SALD}
  & \textcolor{blue}{\textbf{31.63}} 
  & \textcolor{blue}{\textbf{25.49}} 
  & \textcolor{blue}{\textbf{0.7591}}
  & \textcolor{blue}{\textbf{10.85}} 
  & \textcolor{blue}{\textbf{31.84}} 
  & \textcolor{blue}{\textbf{0.7585}} \\
 
\bottomrule
\end{tabular}
}
\vspace{-4mm}
\end{wraptable}
%%%%%%%%%%%%%%%%%%% 

We conduct extensive experiments across $4$ architectures, $7$ benchmarks, and $2$ resolutions to comprehensively evaluate the generality and effectiveness of our method.

\textbf{Datasets.}
We evaluate scalability and precision on seven image--text datasets at $512\!\times\!512$ and $1024\!\times\!1024$ resolutions. At $512\!\times\!512$, MultiGen-20M~\cite{limingcv2024multigen} and COCO (2017)~\cite{coco} assess large-scale performance: Fr\'echet Inception Distance (FID)~\cite{heusel2017gans} is computed on the COCO validation set; MultiGen-trained models are tested on COCO validation for cross-dataset generalization, while COCO-trained models use the same split. Fine-grained evaluation uses Flickr8k~\cite{flicker8k}, CUB-200-2011~\cite{cub200,cub-cap}, and Oxford Flowers~\cite{oxfordflower,flowers-cap} datasets.
At $1024\!\times\!1024$ resolution we use MAGICK~\cite{magick} and DALL$\cdot$E 3 1M~\cite{Egan_Dalle3_1_Million_2024}.
These $5$ datasets use an 80/20 train--test split.
% All others use an 80/20 split.

%%%%%%%%%%%%%%%%%%%%%%%%%%%%%%%%%%
\begin{figure}[t]
    \centering
    \vspace{3mm}
    \begin{minipage}{1\linewidth}
        \centering

        \includegraphics[width=1\linewidth]{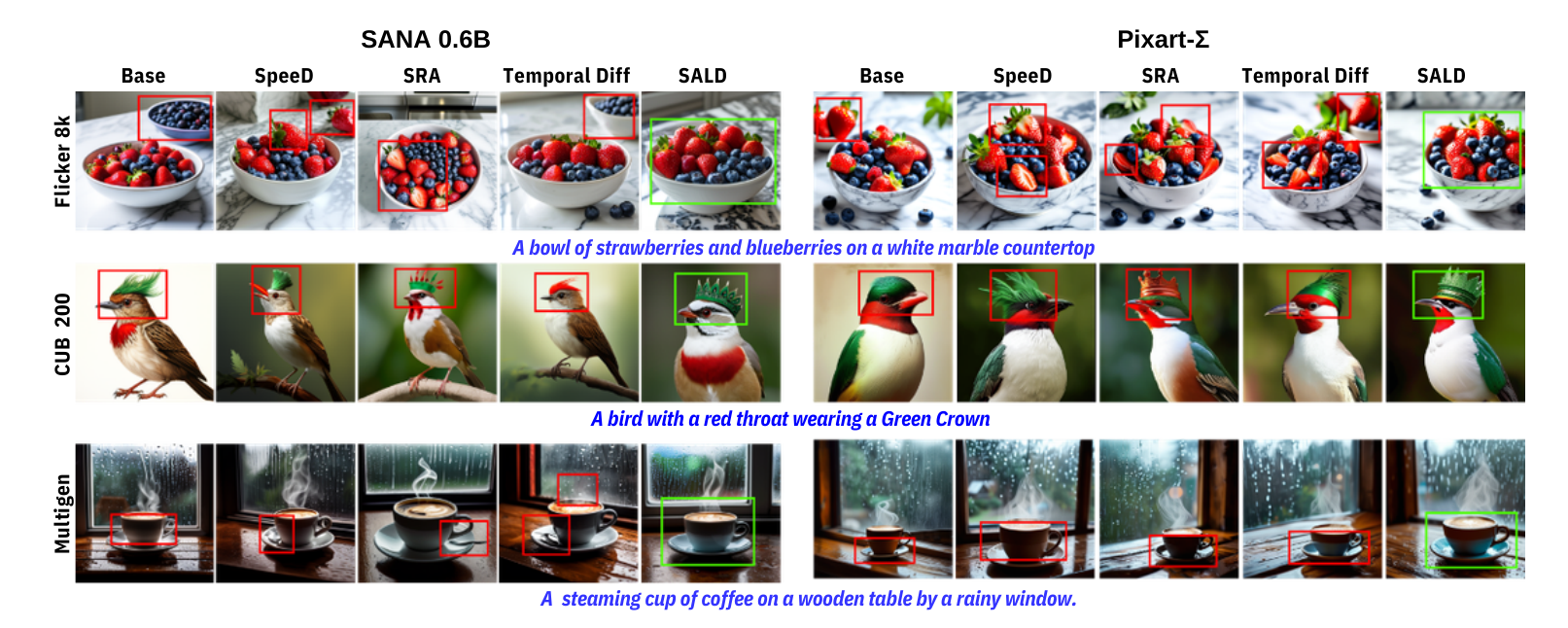}
        \vspace{-5mm}
    \caption{Visual comparison SALD with baselines. \textcolor{red}{\textbf{Red boxes}} indicate missing details or artifacts in latest baseline models. \textcolor{green}{\textbf{Green boxes}} shows our method generates better high-quality details.}
    \label{fig:slad}
    \end{minipage}

    \vspace{3mm}

    \begin{minipage}{1\linewidth}
        \centering
        \includegraphics[width=1\linewidth]{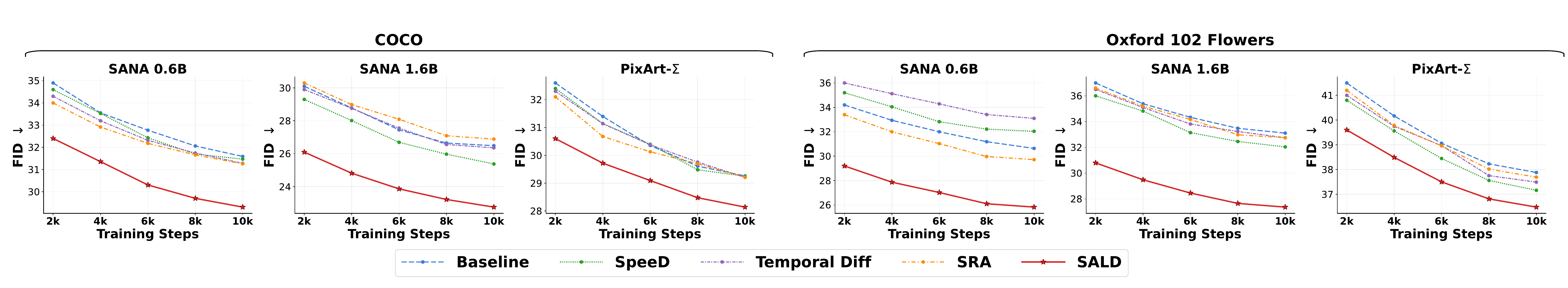}
        \vspace{-5mm}
        \caption{FID with training steps showing better and smoother convergence of SALD w.r.t. baselines.}
        \label{fig:converge}
    \end{minipage}

    \vspace{-5mm}
\end{figure}
%%%%%%%%%%%%%%%%%%%%%%%%%%%%%%%%%%
\begin{figure}[t]
    \centering
    % \vspace{-4mm}
    \includegraphics[width=\linewidth, height=6.5cm]{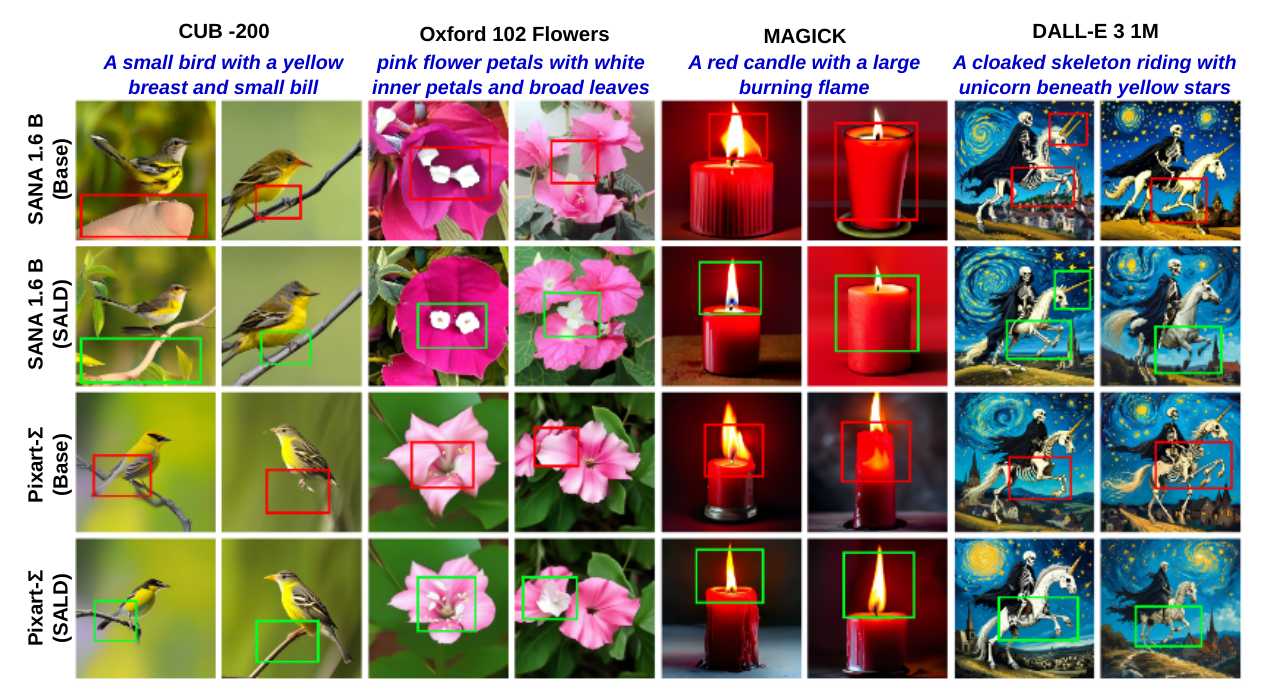}
    % \vspace{-3mm}
        \caption{Visual comparison of generated samples. 
        \textcolor{red}{\textbf{Red boxes}} indicate missing details or artifacts in latest baseline models. 
        \textcolor{green}{\textbf{Green boxes}} highlight that our method generates better high-quality details. 
        We use the same seed for all comparisons to ensure fairness. Captions are in \textcolor{blue}{\textbf{\textit{Blue Color}}.}}
        \label{fig:figure1}
    \vspace{-6mm}
\end{figure}
%%%%%%%%%%%%%%%%%%%%%%%%%%%%%%%%%%%%%%%%%%%

\textbf{Implementation Details.}
All experiments are conducted on NVIDIA RTX A6000 GPUs under a unified setup using the AdamW optimizer~\cite{Adamw} and mixed-precision (\texttt{bf16}) training. We evaluate SALD on four architectures: SANA (0.6B and 1.6B)~\cite{SANA}, PixArt-$\Sigma$~\cite{pixartsigma}, and Flux.2 Klein Base (4B)~\cite{flux2kleinbase2025}. SANA and PixArt-$\Sigma$ models are used for $512$ resolution experiments, while the high-resolution variants of SANA (0.6B and 1.6B), Flux.2 Klein Base (4B) and PixArt-$\Sigma$ are used for $1024$ resolution experiments. All models are trained for $10$k steps unless otherwise specified. Additional training details are mentioned in Appendix~\ref{sec:training_config}.

% Although SALD is developed for flow matching, its relative advantage formulation is architecture-agnostic; we further extend it to DDPM-based models, with details provided in Appendix~\ref{sec:ext_ddpm}.

\textbf{Evaluation Protocol.}
We evaluate generated images using Fr\'{e}chet Inception Distance (FID)~\cite{FID}, CLIP score~\cite{clip}, and Learned Perceptual Image Patch Similarity (LPIPS)~\cite{LPIPS}. FID measures distributional similarity to real images (quality and diversity). CLIP score captures alignment with text prompts, while LPIPS assesses perceptual similarity, together evaluating fidelity, semantics, and realism. FID is computed using the Inception-V3 network with 2048-dimensional pool features. CLIP Score is computed using the CLIP ViT-L/14 model from OpenAI. LPIPS is computed using a VGG backbone. Further details are mentioned in Appendix~\ref{sec:training_config}.

\vspace{-3mm}
%%%%%%%%%%%%%%%%%%%%%%%%%%%%%%%%%%
\subsection{Results and Discussion}
\label{sec:results}
\vspace{-2mm}
%%%%%%%%%%%%%%%%%%%%%%%%%%%%%%%%%%

\noindent\textbf{Quantitative Analysis:}
Tables~\ref{tab:results_large}, and \ref{tab:results1024} report quantitative comparisons across seven benchmarks at $512$ and $1024$ resolutions. SALD consistently outperforms the competing acceleration and distillation baselines, including SpeeD \cite{Speed}, Temporal Diff \cite{temporaldiff}, and SRA \cite{sra}, across both SANA, PixArt and Flux.2 Klein architectures. In particular, SALD achieves the best FID, CLIP, and LPIPS scores across the majority of evaluated settings, demonstrating its ability to simultaneously improve generation fidelity, text-image alignment, and perceptual quality. The improvements are consistent across different model scales and datasets, with particularly substantial FID reductions observed at both $512$ and $1024$ resolutions. Compared with the existing baselines, these results demonstrate that SALD provides a more effective and consistent strategy for improving diffusion model generation quality while preserving semantic alignment.

\noindent\textbf{Qualitative Analysis:}
Figures~\ref{fig:slad} and~\ref{fig:figure1} present qualitative comparisons of SALD across diverse architectures, datasets, and training methods. Figure~\ref{fig:slad} compares SALD with Base, SpeeD, SRA, and Temporal Diff on SANA 0.6B and PixArt-$\Sigma$, where SALD consistently preserves finer details, clearer structures, and prompt-specific semantic attributes, while competing methods often exhibit missing details and visual artifacts. Figure~\ref{fig:figure1} further compares Base and SALD on CUB-200, Oxford-102 Flowers, MAGICK, and DALL$\cdot$E~3 1M, demonstrating improved structural coherence, fine-grained textures, and semantic fidelity, particularly in challenging high-resolution examples. Overall, these results highlight the consistent visual improvements achieved by SALD across models and datasets.

\noindent\textbf{Convergence Analysis:}
Figure~\ref{fig:converge} shows FID convergence across training steps on COCO and Oxford 102 Flowers. Across all settings, SALD converges faster and achieves consistently lower FID than the other baselines, indicating improved optimization and fidelity. Further experiments on this are present in Appendix~\ref{appen:convergence}

\begin{wraptable}{r}{0.65\linewidth}
\vspace{-10mm}
% \vspace{-12.5mm}
\centering
\setlength{\tabcolsep}{2pt}
\Large
\caption{Ablation study of \emph{SALD} on SANA~0.6B across COCO~2017 ($512$) and MAGICK ($1024$). 
% Best result per metric is in \textcolor{blue}{Blue}
}
\vspace{-1mm}
\label{tab:ablation}
\resizebox{\linewidth}{!}{
\begin{tabular}{l|lccc|ccc}
\toprule
\multirow{2}{*}{\textbf{Cases}} & \multirow{2}{*}{\textbf{Setting}}
  & \multicolumn{3}{c|}{\textbf{COCO ($512$)}}
  & \multicolumn{3}{c}{\textbf{MAGICK ($1024$)}} \\
\cmidrule(lr){3-5}\cmidrule(lr){6-8}
 & & \textbf{FID}$\downarrow$ & \textbf{CLIP}$\uparrow$ & \textbf{LPIPS}$\downarrow$
   & \textbf{FID}$\downarrow$ & \textbf{CLIP}$\uparrow$ & \textbf{LPIPS}$\downarrow$ \\
\midrule
\multirow{3}{*}{\textbf{Case 1}}
& AGD only   & 30.95& 26.30& 0.7462 &37.75 & 26.73& 0.7354\\
& SAD only   & 31.02& 26.49& 0.7427 &37.86 & 26.70& 0.7314\\
& TAM only   & 30.98& 26.50 & 0.7474& 37.35& 26.74& 0.7323\\
\midrule
\multirow{3}{*}{\textbf{Case 2}}
& AGD + TAM  & 30.78 & 26.60 &0.7529 & 37.34& 26.72& 0.7396\\
& AGD + SAD  & 30.56 & 26.66 & 0.7375 & 37.78& 26.79& 0.7400\\
& TAM + SAD  & 30.36 & 26.86 & 0.7391 & 37.23& 26.78& 0.7389\\
\midrule
\multirow{4}{*}{\textbf{Case 3}}
& $\tau = 0.1$                 & 31.10 & 26.09 & 0.7459 & 38.22 & 26.66 & 0.7423\\
& $\tau = 0.5$ \textit{(ours)} & 29.31 & 27.06& 0.7374 & 36.74 & 26.85& 0.7230 \\
& $\tau = 1.0$                 & 31.22& 26.08 & 0.7455 & 38.34& 26.63& 0.7445\\
& $\tau = 2.0$                 & 31.25 & 26.13 & 0.7457 & 38.42& 26.62& 0.7431\\
\midrule
\multirow{4}{*}{\textbf{Case 4}}
& $\alpha = 0.50$                  & 30.89 & 26.14 & 0.7454 & 38.96& 26.64& 0.7443\\
& $\alpha = 0.90$ \textit{(ours)}  & 29.31 & 27.06& 0.7474 & 36.74 & 26.85& 0.7230 \\
& $\alpha = 0.99$                  & 29.52 & 26.11& 0.7461 & 38.23& 26.64& 0.7443\\
& $\alpha = 0.999$                 & 29.30 & 26.12 & 0.7463 & 38.76& 26.68& 0.7455\\
\midrule
\multirow{3}{*}{\textbf{Case 5}}
& $\Delta_{\max} = 100$                & 30.17& 26.75 & 0.7393& 36.78 & 26.62& 0.7467\\
& $\Delta_{\max} = 250$ \textit{(ours)}& 29.31 & 27.06& 0.7374 & 36.74 & 26.85& 0.7230 \\
& $\Delta_{\max} = 500$                & 31.96&26.64 & 0.7398& 36.81& 26.61 &0.7405 \\
\midrule
\multirow{5}{*}{\textbf{Case 6}}
& Abs.\ loss (no contrast)     & 29.64 & 26.98 & 0.7385 & 36.95& 26.54&0.7305\\
& Fixed $\delta$ (no ramp)     & 29.79 & 26.77 & 0.7383 & 37.05& 26.48 & 0.7475\\
& Detach disabled              & 30.05&26.85 & 0.7390 & 37.43& 26.67& 0.7445\\
& $12$ Steps                   & 30.63& 26.68 & 0.7472 & 37.85& 26.62&0.7476 \\
& Cross-Batch Advantage        & 30.18& 26.92 & 0.7386 & 37.57& 26.75& 0.7434\\
\midrule
\textbf{Base} & Uniform & 31.59 & 26.80 & 0.7408 & 38.59 & 26.68 & 0.7470\\
\midrule
\rowcolor{blue!15} \textbf{SALD} & \textbf{AGD-TAM-SAD} & 29.31 & 27.06& 0.7374 & 36.74 & 26.85& 0.7230\\
\bottomrule
\end{tabular}}
\vspace{-7mm}
\end{wraptable}
% \vspace{-2mm}
%%%%%%%%%%%%%%%%%%%%%%%%%%%%%%%%%%

\vspace{-3mm}
\subsection{Ablation Study}
\label{sec:ablations}
\vspace{-2mm}

We conduct an ablation study using SANA~0.6B on COCO~2017 ($512\!\times\!512$) and MAGICK ($1024\!\times\!1024$), using the uniform flow-matching baseline as the reference across all cases.

\noindent\textbf{\textit{Case 1:} Individual Contributions.}
AGD contributes most strongly among individual components, as direct difficulty-aware reweighting provides an immediate training signal. SAD offers complementary gains through spectral routing, improving structural fidelity at the cost of modest CLIP degradation. TAM alone yields the weakest improvement, since its EMA-based memory requires a sustained difficulty signal from AGD to stabilise, a notable limitation when used in isolation. Together, all three components are necessary to close the full performance gap.

\noindent\textbf{\textit{Case 2:} Pairwise Combinations.}
AGD+SAD achieves the best pairwise result, as the two components operate on complementary axes, sample difficulty, and spectral structure, without cross-dependency. TAM+SAD performs weakest among pairs, confirming TAM's dependence on AGD's advantage signal for effective curriculum updates. The SALD surpasses all pairs, demonstrating non-redundant synergy across components.

\noindent\textbf{\textit{Case 3:} AGD Temperature $\tau$.}
Moderate sharpness ($\tau=0.5$) best balances gradient focus and perceptual stability. Overly low $\tau$ over-concentrates on extreme samples, slightly harming LPIPS, while high $\tau$ degrades toward a near-uniform loss, weakening the difficulty signal. SALD exhibits a clear U-shaped sensitivity, with gains degrading at both extremes.

\noindent\textbf{\textit{Case 4:} TAM EMA Decay $\alpha$.}
The default $\alpha=0.90$ strikes the best balance between curriculum responsiveness and stability. Aggressive decay ($\alpha=0.50$) introduces noisy difficulty estimates, while near-frozen memory ($\alpha=0.999$) fails to track the evolving training landscape. The wider performance spread on MAGICK confirms that longitudinal difficulty tracking is more critical at higher resolution, where sample complexity varies more sharply.

\noindent\textbf{\textit{Case 5:} Maximum Noise Gap $\Delta_{\max}$.}
Both extremes degrade performance: a narrow gap ($\Delta_{\max}=100$) weakens contrast between hard and easy paths, while an overly wide gap ($\Delta_{\max}=500$) pushes the easy path into a near-noiseless regime, overestimating difficulty and destabilising training. The default $\Delta_{\max}=250$ yields the best results across both benchmarks, with a consistent U-shaped FID response confirming the sensitivity of the advantage signal to gap magnitude.

% \noindent\textbf{\textit{Case 6:} Core Design Decisions.}
% All four design choices prove critical. \textit{Absolute loss (no contrast):} replacing $A^{(i)} = \ell_b^{(i)} - \mathrm{sg}(\ell_a^{(i)})$ with $|\ell_a^{(i)}|$ collapses gains toward the baseline, most severely at $1024^2$, confirming that relative difficulty, not raw loss magnitude, is the effective training signal. \textit{Fixed $\delta$ (no curriculum ramp):} removing the warm-up destabilises the advantage signal during early training, mildly degrading perceptual quality. \textit{Detach disabled:} allowing $w^{(i)}$ to backpropagate causes the model to exploit the weighting mechanism rather than improve predictions, validating $\mathrm{sg}(\cdot)$ as a non-negotiable design constraint. \textit{Reduced inference steps ($12$ vs.\ $20$):} notably, RAD at 12 steps still outperforms the 20-step uniform baseline across both benchmarks, suggesting that by training the model to better approximate the hard-to-easy denoising trajectory, RAD produces a smoother learned flow that requires fewer inference steps to converge, offering a practical efficiency benefit without sacrificing quality.

\noindent\textbf{\textit{Case 6:} Core Design Decisions.}
All five choices prove critical. \textit{Absolute loss (no contrast):} replacing $A^{(i)} = \ell_b^{(i)} - \mathrm{sg}(\ell_a^{(i)})$ with $|\ell_a^{(i)}|$ collapses gains toward the baseline, most severely at $1024^2$, confirming relative difficulty, not raw loss magnitude, as the effective signal. \textit{Fixed $\delta$ (no ramp):} removing the warm-up destabilises the advantage signal during early training, mildly degrading perceptual quality. \textit{Detach disabled:} allowing $w^{(i)}$ to backpropagate causes the model to exploit the weighting mechanism rather than improve predictions, validating $\mathrm{sg}(\cdot)$ as non-negotiable. \textit{Reduced steps ($12$ vs.\ $20$):} SALD at 12 steps still outperforms the 20-step uniform baseline, indicating that training against the hard-to-easy trajectory yields a smoother learned flow with fewer steps required at inference. \textit{Cross-batch advantage:} computing advantage relative to a second independently drawn batch rather than the same sample's easier noise path yields partial gains but falls short of SALD, confirming that intra-sample noise contrast is a strictly stronger signal than inter-sample difficulty, which conflates sample hardness with distributional variation.
\vspace{-4mm}
\section{Related Work}
\vspace{-3mm}
\label{sec:related}
Denoising diffusion models~\cite{ho2020denoising,das2026unityattentionflownetworks} and score-based generative models~\cite{song2021scorebased} established iterative denoising as a powerful paradigm for image synthesis. Flow matching~\cite{lipman2023flow,albergo2023stochastic} provides
a continuous-time alternative by learning velocity fields along interpolation paths. Recent transformer-based generators, including DiT~\cite{peebles2023scalable}, rectified-flow transformers~\cite{esser2024scaling}, SANA~\cite{SANA},
PixArt-$\alpha$~\cite{pixartalpha}, and PixArt-$\Sigma$~\cite{pixartsigma}, further improve scalability and high-resolution generation. However, most objectives still
treat samples, timesteps, and spatial locations uniformly~\cite{DAS2025113152,11491926,karim2026focusforcingincontextobject}. Several works improve training by reallocating the gradient budget across noise levels. EDM~\cite{karras2022elucidating} studies principled noise schedules and $\sigma$-weighting, while P2~\cite{choi2022perception} emphasizes perceptually important SNR regions. Improved noise schedules~\cite{hang2024noise} and noise-instance analysis~\cite{chen2024noises} further show that different noise levels and noise samples provide unequal training signals~\cite{das2026efficienttextguidedconvolutionaladapter,9577942}. These approaches define global timestep-dependent weights, whereas AGD computes per-sample relative difficulty from the loss gap between low and high-noise paths.

Our work is also related to curriculum learning~\cite{bengio2009curriculum}, self-paced learning~\cite{kumar2010selfpaced}, dynamic instance hardness
~\cite{chang2020dynamicinstancehardness}, adaptive curricula~\cite{kong2021adaptivecl}, diffusion curricula~\cite{liang2025discl}, and adaptive point-weighting
~\cite{apwcl2025}. Hard-example mining~\cite{felzenszwalb2010dpm,shrivastava2016ohem}
and DRO~\cite{sagawa2019dro,nips2024droalignment,icml2024dro_layers} similarly emphasize difficult examples, while meta-hard mining~\cite{mu2025metahard} and coreset selection~\cite{icml2025coreset} select informative samples. Unlike these methods, TAM stores an EMA of relative advantage and uses it to adapt the timestep gap for each sample. Finally, SAD connects to spectral bias studies~\cite{rahaman2019spectral,
zhiqin2020frequency,wang2025spectraldiffusion} and frequency-aware diffusion methods
~\cite{nips2025spectralanalysis,si2024freeu,shapingbias2025,spectralguided2026,
liu2025freqdehaze}.
Instead of modifying the architecture, sampler, or noising process, SAD uses contrastive spectral residuals to construct spatial importance weights for loss reweighting. Recent work uses feedback-conditioned output distribution to distill
\cite{shenfeld2026selfdistillation,hubotter2026reinforcement} from teacher to student in the language domain but requires demonstration example and EMA based teacher. The proposed model can be seen as distillation from easy to hard sample, however it does not requires extra copy of the model parameters.

\vspace{-4mm}
\section{Conclusion}
\label{sec:conclusion}
\vspace{-3mm}

% We presented SALD, a unified training framework for flow-based diffusion models that reallocates gradient budget across samples, noise levels, and spatial frequencies. AGD estimates relative sample difficulty by contrasting losses between low- and high-noise states, producing normalized within-batch importance weights. TAM extends this signal across the training set through persistent per-sample memory, enabling an adaptive curriculum over timestep gaps. SAD further refines optimization by converting contrastive spectral residuals into spatial importance weights, directing gradients toward under-learned image regions. Since all weighting signals are detached, SALD preserves the original prediction target while improving where and how training effort is applied. In general, SALD offers a simple, architecture-agnostic approach for improving diffusion training without modifying the model, sampler, or inference procedure. Extensive experiments across diverse datasets and architectures demonstrate the effectiveness of the proposed framework, while theoretical analysis supports the empirical findings.

We presented SALD, a unified training framework for flow-based diffusion models that reallocates gradient budget across samples, noise levels, and frequency components. AGD estimates relative sample difficulty by contrasting losses between low and high-noise states, producing normalized within-batch importance weights. TAM extends this signal across the training set through persistent per-sample memory, enabling an adaptive curriculum over timestep gaps. SAD further refines optimization by converting contrastive spectral residuals into frequency-derived latent weights, emphasizing frequency patterns that remain under-learned. Since only the easy-path reference is detached, SALD preserves the original prediction target while improving where and how training effort is applied. In general, SALD offers a simple, architecture-agnostic approach for improving diffusion training without modifying the model, sampler, or inference procedure. Extensive experiments across diverse datasets and architectures demonstrate the effectiveness of the proposed framework, while theoretical analysis supports the empirical findings.
%\newpage

% \input{sec/xtra}
\bibliographystyle{iclr2027_conference}
\bibliography{main.bib}

% Supplementary
%%%%%%%%%%%%%%%%%%%%%%%%%%%%%%%%%%%%%%%%
\newpage
\appendix
\section*{Appendix}
\label{sec:suppl}

%%%%%%%%%%%%%%%%%%%%%%%%%%%%%%%%%%%%%%%%%%%%%%%%%%%
\section{Why Is Uniform Training Suboptimal?}
%%%%%%%%%%%%%%%%%%%%%%%%%%%%%%%%%%%%%%%%%%%%%%%%%%%
\label{app:uniform-training}

The standard flow-matching loss used in Section~2.1 averages the losses of all
training samples and all elements of their latent residuals.  Apart from the
time-dependent factor $w(t)$, it assigns the same explicit coefficient to every
sample and latent element.  This does not mean that all samples produce equal
gradients; samples with larger residuals can naturally produce larger gradient
magnitudes.  Rather, the objective has no explicit mechanism for adapting a
sample's importance according to its current difficulty or for distinguishing
the frequency content of its remaining residual.  This uniform treatment leads
to two limitations that can become more important as training progresses,
especially during domain-specific fine-tuning.

The first limitation is at the sample level.  Not all images are equally
difficult for the current model.  A clean portrait against a plain background
may be learned early, after which its loss becomes small and further updates
provide limited new information.  Meanwhile, a cluttered street scene with
reflections, thin text, and overlapping objects may remain difficult for much
longer.  Under uniform weighting, both samples retain the same explicit
coefficient in the mini-batch objective, even though reallocating more emphasis
to the difficult sample may be more useful.  Prior work on hard-example mining
and diffusion-noise selection similarly shows that examples and noise
realizations can provide unequal training value
\cite{shrivastava2016ohem,qi2024notallnoises}.  Consequently, uniform weighting
can spend part of the finite optimization budget on already well-learned
samples while under-emphasizing examples that remain difficult.

The second limitation concerns the frequency content of the residual within
each image.  Deep networks are known to learn low-frequency content faster than
high-frequency content, a behavior referred to as spectral bias or the
frequency principle \cite{rahaman2019spectral,zhiqin2020frequency}.  Related
frequency-dependent behavior has also been studied in diffusion models
\cite{yang2023diffusion,si2024freeu,wang2025spectral}.  Smooth structures such
as skies and walls may be reconstructed earlier, whereas fine details such as
hair strands, feather textures, specular highlights, and small text can lag
behind.  However, the standard mean-squared error collapses the full latent
residual into a single scalar and does not explicitly indicate which frequency
bands retain comparatively large error.  Importantly, frequency magnitude does
not identify the spatial locations of those errors because the power spectrum
discards Fourier phase.  The relevant goal is therefore to identify difficult
frequency components and use them to construct a frequency-derived latent
weight, rather than to claim a localized pixel-level error map.

One natural idea is curriculum learning
\cite{bengio2009curriculum,kumar2010selfpaced}, in which easy samples are shown
first and harder samples are introduced later.  Classical curricula, however,
often rely on a fixed notion of difficulty.  In diffusion training, a sample's
difficulty changes as the model improves, so a static ordering can quickly
become stale.  What is needed instead is a strategy that (i) measures which
samples become difficult relative to the current model when noise is increased,
(ii) remembers this relative difficulty across training steps, and (iii)
identifies which residual frequencies remain difficult.  Unlike
demonstration- or feedback-conditioned language-model self-distillation
\cite{shenfeld2026selfdistillation,hubotter2026reinforcement}, SALD constructs
its reference from two noise states of the same image--caption pair: a
lower-noise easy path evaluated without gradient tracking and a higher-noise
hard path that supplies the training gradient.  Advantage-Guided Diffusion
(AGD) performs differentiable sample-level reweighting, Temporal Advantage
Memory (TAM) maintains persistent relative-difficulty estimates and adapts
future noise gaps, and Spectral Advantage Decomposition (SAD) constructs a
frequency-derived latent weight.  These components share one set of model
parameters and require no external teacher or additional model copy during
training or inference.

%%%%%%%%%%%%%%%%%%%%%%%%%%%%%%%%%%%%%%%%
\section{Extension to DDPM}\label{sec:ext_ddpm}
%%%%%%%%%%%%%%%%%%%%%%%%%%%%%%%%%%%%%%%%
Our framework can be extended to DDPM-style noise-prediction models, although we do not validate this extension experimentally (Appendix~\ref{sec:limitations}).  The core
difference between flow matching and DDPM lies in the prediction target
\cite{lipman2023flow,ho2020denoising}.  In the flow-matching formulation used
in the main paper, the network predicts the velocity
$u=\epsilon-x_0$ along the linear interpolation path.  In an
$\epsilon$-prediction DDPM, the network predicts the Gaussian noise added to
the clean representation.  The forward process and residual are
\begin{equation}
    x_t
    =\sqrt{\bar{\alpha}_t}\,x_0
    +\sqrt{1-\bar{\alpha}_t}\,\epsilon,
    \qquad
    r(t,\epsilon)
    =\epsilon_\theta(x_t,t,c)-\epsilon,
    \label{eq:ddpm-forward-residual}
\end{equation}
where $\epsilon\sim\mathcal{N}(0,I)$ and $\bar{\alpha}_t$ is the cumulative
noise schedule.  The corresponding scalar loss is
\begin{equation}
    \ell(t,\epsilon)
    =\frac{w(t)}{D}\left\|r(t,\epsilon)\right\|_2^2,
    \qquad D=CHW.
    \label{eq:ddpm-scalar-loss}
\end{equation}
Thus, extending SALD requires replacing the velocity residual in the main
formulation with the DDPM noise-prediction residual in
Eq.~\eqref{eq:ddpm-forward-residual}.

The two paths keep the same roles as in the main paper.  Path A is the easy,
lower-noise path and is evaluated without gradient tracking; Path B is the
hard, higher-noise path and supplies the training gradient.  Under the standard
DDPM ordering, $t_{\mathrm e}<t_{\mathrm h}$ and
$\bar{\alpha}_{t_{\mathrm e}}>\bar{\alpha}_{t_{\mathrm h}}$.  Using independent
noises for the two paths, we compute
\begin{equation}
    \ell_{\mathrm e}
    =\ell(t_{\mathrm e},\epsilon_{\mathrm e}),
    \qquad
    \ell_{\mathrm h}
    =\ell(t_{\mathrm h},\epsilon_{\mathrm h}),
    \qquad
    A_i
    =\ell_{\mathrm h,i}
    -\sg\!\left(\ell_{\mathrm e,i}\right).
    \label{eq:ddpm-advantage}
\end{equation}
Nothing about AGD changes except the predicted quantity.  Importantly,
$A_i$ is not fully detached: only the easy reference is stopped, so
$\nabla_\theta A_i=\nabla_\theta\ell_{\mathrm h,i}$.

TAM also keeps the same update rule.  It stores a scalar exponential-moving
average for each training pair and updates it using the numerical value of the
relative advantage.  The memory vector $m$, decay $\alpha$, and modulation by
$\tanh(m_i)$ remain unchanged.  The stored memory is used to select the current
easy-to-hard gap before the two paths are evaluated; the current advantage then
updates the memory for the next visit.  The memory is a non-trainable state
buffer and is not backpropagated through time.

SAD also carries over by operating on the noise-prediction residuals
$r_{\mathrm e}^{(i)}$ and $r_{\mathrm h}^{(i)}$ instead of velocity residuals.
These tensors remain in $\mathbb{R}^{C\times H\times W}$, so the real FFT,
power-spectrum comparison, spectral advantage, sigmoid gate, and inverse-FFT
construction of $M_i$ are unchanged.  As in the flow-matching formulation,
only the easy residual spectrum is stopped; $M_i$ remains differentiable
through the hard residual.  Frequency-dependent learning behavior motivates
this comparison \cite{rahaman2019spectral,wang2025spectral}, but its
interpretation remains limited: because the power spectrum discards phase,
$M_i$ is a frequency-derived latent weight rather than a localized error map.
Moreover, applying $M_i$ in the latent domain does not independently rescale
each Fourier coefficient.

The remaining practical consideration is loss weighting.  The distribution
used to sample $t$ and the factor $w(t)$ used to weight the loss are distinct
choices.  The original DDPM simple objective uses uniform loss weighting,
whereas perceptual or signal-to-noise-ratio-derived alternatives can also be
used \cite{ho2020denoising,choi2022perception}.  SALD is compatible with any
nonnegative rule $w(t)$, provided that the same rule is evaluated at the
corresponding easy and hard levels when computing
Eq.~\eqref{eq:ddpm-advantage}.  The final SAD-weighted loss uses
$w(t_{\mathrm h})$ because gradients are taken only through the hard path.

%%%%%%%%%%%%%%%%%%%%%%%%%%%%%%%%%%%%%%%%%%%%%%%%%%%
\section{Proof and Scope of Theorem~1}
%%%%%%%%%%%%%%%%%%%%%%%%%%%%%%%%%%%%%%%%%%%%%%%%%%%

For a fixed TAM memory state $m$, let
\begin{equation}
    F_m(\theta)
    :=
    \E_{\xi}\!\left[
        \widehat{\mathcal L}_{\mathrm{SALD}}(\theta;\xi,m)
    \right],
    \label{eq:population-objective}
\end{equation}
where $\xi$ denotes all randomness used to construct a mini-batch,
including the sampled training examples, noise levels, and Gaussian noises.
Because the normalized AGD weights couple the examples through their batch
denominator, $F_m$ is the population objective induced by an entire random
mini-batch; it is not, in general, an expectation of independent
single-example objectives. The unbiasedness assumption below is satisfied by
Eq.~\eqref{eq:population-objective} whenever differentiation and expectation
can be interchanged.

\begin{theorem}[Conditional stationarity bound for the differentiable SALD objective]
Fix $m$. Assume that $F_m$ is lower bounded by $F_m^\star$ and has an
$L$-Lipschitz gradient. At iteration $s$, let
\begin{equation}
    g_s
    =
    \nabla_\theta
    \widehat{\mathcal L}_{\mathrm{SALD}}(\theta_s;\xi_s,m)
\end{equation}
be the stochastic gradient containing all gradient paths through the AGD
weights $\omega_i$ and the SAD kernels $M_i$. Suppose
\begin{align}
    \E[g_s\mid\theta_s]
    &=\nabla F_m(\theta_s),
    \label{eq:unbiased-gradient}\\
    \E\!\left[
        \norm{g_s-\nabla F_m(\theta_s)}_2^2
    \right]
    &\leq \nu^2.
    \label{eq:bounded-variance}
\end{align}
For the update
\begin{equation}
    \theta_{s+1}=\theta_s-\eta g_s,
    \qquad 0<\eta\leq\frac{1}{L},
\end{equation}
and a deterministic initial point $\theta_0$,
\begin{equation}
    \frac{1}{J}
    \sum_{s=0}^{J-1}
    \E\norm{\nabla F_m(\theta_s)}_2^2
    \leq
    \frac{2\bigl(F_m(\theta_0)-F_m^\star\bigr)}{\eta J}
    +L\eta\nu^2.
    \label{eq:stationarity-bound}
\end{equation}
\end{theorem}

\begin{proof}
Since $F_m$ is $L$-smooth, the descent lemma gives
\begin{align}
    F_m(\theta_{s+1})
    &\leq
    F_m(\theta_s)
    +\left\langle
        \nabla F_m(\theta_s),
        \theta_{s+1}-\theta_s
    \right\rangle
    +\frac{L}{2}\norm{\theta_{s+1}-\theta_s}_2^2
    \\
    &=
    F_m(\theta_s)
    -\eta\left\langle\nabla F_m(\theta_s),g_s\right\rangle
    +\frac{L\eta^2}{2}\norm{g_s}_2^2.
    \label{eq:descent-lemma-step}
\end{align}
By conditional unbiasedness,
\begin{equation}
    \E\!\left[
        \left\langle\nabla F_m(\theta_s),g_s\right\rangle
    \right]
    =
    \E\norm{\nabla F_m(\theta_s)}_2^2.
    \label{eq:inner-product-expectation}
\end{equation}
To control the second moment, define

\begin{equation}
    \zeta_s:=g_s-\nabla F_m(\theta_s).
\end{equation}

Equation~\eqref{eq:unbiased-gradient} implies
$\E[\zeta_s\mid\theta_s]=0$. Therefore, the cross term vanishes and
Eq.~\eqref{eq:bounded-variance} yields
\begin{align}
    \E\norm{g_s}_2^2
    &=
    \E\norm{\nabla F_m(\theta_s)+\zeta_s}_2^2
    \\
    &=
    \E\norm{\nabla F_m(\theta_s)}_2^2
    +\E\norm{\zeta_s}_2^2
    \\
    &\leq
    \E\norm{\nabla F_m(\theta_s)}_2^2+\nu^2.
    \label{eq:second-moment-bound}
\end{align}
Taking expectations in Eq.~\eqref{eq:descent-lemma-step} and applying
Eqs.~\eqref{eq:inner-product-expectation}--\eqref{eq:second-moment-bound}
gives
\begin{align}
    \E F_m(\theta_{s+1})
    &\leq
    \E F_m(\theta_s)
    -\eta\left(1-\frac{L\eta}{2}\right)
    \E\norm{\nabla F_m(\theta_s)}_2^2
    +\frac{L\eta^2}{2}\nu^2.
    \label{eq:one-step-descent}
\end{align}
Because $\eta\leq 1/L$,
$1-L\eta/2\geq 1/2$. Hence,
\begin{equation}
    \frac{\eta}{2}
    \E\norm{\nabla F_m(\theta_s)}_2^2
    \leq
    \E F_m(\theta_s)-\E F_m(\theta_{s+1})
    +\frac{L\eta^2}{2}\nu^2.
    \label{eq:rearranged-descent}
\end{equation}
Summing Eq.~\eqref{eq:rearranged-descent} from $s=0$ to $J-1$
telescopes the objective values:
\begin{equation}
    \frac{\eta}{2}
    \sum_{s=0}^{J-1}
    \E\norm{\nabla F_m(\theta_s)}_2^2
    \leq
    F_m(\theta_0)-\E F_m(\theta_J)
    +\frac{JL\eta^2}{2}\nu^2.
\end{equation}
Using $F_m(\theta_J)\geq F_m^\star$, dividing by
$\eta J/2$, and rearranging gives
Eq.~\eqref{eq:stationarity-bound}.
\end{proof}

\begin{remark}[Meaning of the bound]
For a constant stepsize and nonzero stochastic variance, the right-hand side
of Eq.~\eqref{eq:stationarity-bound} approaches the nonzero error floor
$L\eta\nu^2$. Thus, the displayed inequality guarantees convergence to a
stationarity neighborhood, not necessarily to an exact stationary point. A
vanishing average gradient norm follows when $\nu=0$, or when the stepsize
is decreased appropriately. For example, choosing
$\eta=\Theta(J^{-1/2})$, while respecting $\eta\leq1/L$, gives the
standard $\mathcal O(J^{-1/2})$ stationarity rate.
\end{remark}

\begin{remark}[Fixed-memory scope]
The telescoping argument uses one fixed objective $F_m$. During the full
SALD procedure, TAM changes the memory from $m_s$ to $m_{s+1}$, so the
corresponding objective may also change from $F_{m_s}$ to $F_{m_{s+1}}$.
Theorem~1 therefore provides a convergence result conditional on a fixed TAM
state. Extending it to the complete online-memory dynamics would require an
additional assumption controlling the objective drift caused by each memory
update, or a separate analysis of the joint state $(\theta_s,m_s)$.
\end{remark}

\begin{remark}[What is not implied]
The theorem assumes smoothness, unbiasedness, and bounded variance for the
actual differentiable SALD gradient, including the paths through $\omega_i$
and $M_i$. It neither proves these assumptions for a particular neural
network nor shows that SALD converges faster than uniform flow matching. It
also does not imply that the SALD and baseline objectives have the same
stationary points.
\end{remark}

%%%%%%%%%%%%%%%%%%%%%%%%%%%%%%%%%%%%%%%%
\section{Additional Results}
\label{sec:addi_res}
%%%%%%%%%%%%%%%%%%%%%%%%%%%%%%%%%%%%%%%%

In this section, we present additional qualitative and quantitative results
to further analyze SALD's behavior across different datasets, architectures,
and training settings. We include extended visual comparisons and
supplementary experiments that highlight the contributions of AGD, TAM, and
SAD to generation quality, optimization stability, and fine-grained detail
preservation.

Algorithm~\ref{alg:sald-training} summarizes a single SALD training step. Given a mini-batch of
image--caption pairs, we first determine a base noise-level gap using a simple
curriculum: during warm-up, the gap increases from $1$ to
$\Delta$, after which it is sampled uniformly from
$\{1,\ldots,\Delta\}$. This base gap is then modulated for each
sample using Temporal Advantage Memory (TAM), which scales the gap according
to the sample's stored difficulty score and clips it to the valid range. For
each example, we sample a lower-noise easy level and obtain the corresponding
higher-noise hard level by adding the TAM-adjusted gap. Two noisy views of the
same latent are then constructed using independent Gaussian perturbations.
The same model performs two forward passes: the easy path is evaluated under
\texttt{no\_grad} and serves as an internal reference, whereas the hard path
remains differentiable and supplies the training gradient. From the residuals
of these two paths, Spectral Advantage Decomposition (SAD) compares their
frequency-domain error magnitudes and maps a zero-phase spectral gate back to
the latent coordinates to obtain a frequency-derived weighting kernel. Since
this construction does not retain the residual phase or recover the locations
of the original errors, the kernel should not be interpreted as a localized
pixel-level advantage map. Instead, it reweights the hard-path residual loss
using the relative spectral errors of the two paths. In parallel,
Advantage-Guided Diffusion (AGD) defines the sample advantage as the
difference between the differentiable hard-path loss and the stop-gradient
easy-path reference loss. It converts this advantage into a normalized sample
weight that emphasizes harder examples within the batch. Finally, TAM stores
the numerical advantage for future visits to the sample; the memory is
non-trainable, and no gradient is propagated through this state update. The
SALD objective is then formed by averaging the sample-weighted SAD losses of
the hard paths before backpropagation. In this way, the algorithm adapts the
training signal across samples, noise levels, and spectral components through
a dynamic, self-referenced curriculum.

%%%%%%%%%%%%%%%%%%%%%%%%%%%%%%%%
\begin{algorithm}[t]
\caption{SALD Training Step}
\label{alg:sald-training}
\small
\begin{algorithmic}[1]
\Require Mini-batch $\{(x_0^{(i)},c^{(i)},\operatorname{idx}^{(i)})\}_{i=1}^{B}$
\Require Noise levels $0\leq\sigma_0<\cdots<\sigma_{T-1}\leq1$; current step $s$; $\tau,\tau_s,\alpha,\lambda$; $K>1$; $1\leq\Delta_{\max}<T$
\Statex \textbf{Global:} non-trainable TAM memory $m\in\mathbb{R}^{N}$ (initialized to zero); $D=CHW$
\Statex \textbf{Notation:} subscript $a$ denotes the easy path ($t_{\mathrm e}=\sigma_{k_a}$) and $b$ the hard path ($t_{\mathrm h}=\sigma_{k_b}$)
\If{$s<K$}
    \State $\delta_{\mathrm{base}}\gets 1+\left\lfloor\dfrac{(\Delta_{\max}-1)s}{K-1}\right\rfloor$
\Else
    \State $\delta_{\mathrm{base}}\sim\operatorname{Uniform}\{1,\ldots,\Delta_{\max}\}$
\EndIf
\For{$i=1,\ldots,B$}
    \State $\delta^{(i)}\gets\clip\!\left(\operatorname{round}\!\left[\delta_{\mathrm{base}}
    \bigl(1+\lambda\tanh(m_{\operatorname{idx}^{(i)}})\bigr)\right],
    1,\Delta_{\max}\right)$ \Comment{TAM gap}
    \State $u^{(i)}\sim\LogitNormal$,
    $\quad k_a^{(i)}\gets\left\lfloor u^{(i)}(T-\delta^{(i)})\right\rfloor$,
    $\quad k_b^{(i)}\gets k_a^{(i)}+\delta^{(i)}$
    \State $\epsilon_a^{(i)},\epsilon_b^{(i)}\sim\mathcal{N}(0,I)$
    \State $z_{k_a}^{(i)}\gets(1-\sigma_{k_a^{(i)}})x_0^{(i)}+\sigma_{k_a^{(i)}}\epsilon_a^{(i)}$
    \State $z_{k_b}^{(i)}\gets(1-\sigma_{k_b^{(i)}})x_0^{(i)}+\sigma_{k_b^{(i)}}\epsilon_b^{(i)}$
    \State $\widehat v_a^{(i)}\gets
    \operatorname{no\_grad}\!\left[v_\theta(z_{k_a}^{(i)},k_a^{(i)},c^{(i)})\right]$
    \Comment{easy path; no gradient}
    \State $\widehat v_b^{(i)}\gets v_\theta(z_{k_b}^{(i)},k_b^{(i)},c^{(i)})$
    \Comment{hard path; gradient on}
    \State $r_a^{(i)}\gets\widehat v_a^{(i)}-(\epsilon_a^{(i)}-x_0^{(i)})$
    \State $r_b^{(i)}\gets\widehat v_b^{(i)}-(\epsilon_b^{(i)}-x_0^{(i)})$
    \State $M^{(i)}\gets\operatorname{SAD}(r_a^{(i)},r_b^{(i)},\tau_s)$
    \Comment{not detached; mean approximately one}
    \State $\ell_{\mathrm{sad}}^{(i)}\gets\dfrac{w(\sigma_{k_b^{(i)}})}{D}
    \sum_{p=1}^{D}M_p^{(i)}\bigl(r_{b,p}^{(i)}\bigr)^2$
    \State $\ell_a^{(i)}\gets\dfrac{w(\sigma_{k_a^{(i)}})}{D}
    \lVert r_a^{(i)}\rVert_2^2$
    \Comment{easy loss; no gradient}
    \State $\ell_b^{(i)}\gets\dfrac{w(\sigma_{k_b^{(i)}})}{D}
    \lVert r_b^{(i)}\rVert_2^2$
    \Comment{hard loss; gradient on}
    \State $A^{(i)}\gets\ell_b^{(i)}-\sg\!\left(\ell_a^{(i)}\right)$
    \Comment{only the easy loss is detached}
\EndFor
\For{$i=1,\ldots,B$}
    \State $\omega^{(i)}\gets\dfrac{B\,\operatorname{sigmoid}(A^{(i)}/\tau)}
    {\sum_{j=1}^{B}\operatorname{sigmoid}(A^{(j)}/\tau)}$
    \Comment{not detached}
\EndFor
\For{$i=1,\ldots,B$}
    \State $m_{\operatorname{idx}^{(i)}}\gets
    \alpha m_{\operatorname{idx}^{(i)}}+(1-\alpha)\,\sg\!\left(A^{(i)}\right)$
    \Comment{numerical value only; no graph}
\EndFor
\State $\mathcal{L}_{\mathrm{SALD}}\gets\dfrac{1}{B}
\sum_{i=1}^{B}\omega^{(i)}\ell_{\mathrm{sad}}^{(i)}$
\State Backpropagate $\nabla_\theta\mathcal{L}_{\mathrm{SALD}}$ and update $\theta$
\end{algorithmic}
\end{algorithm}
%%%%%%%%%%%%%%%%%%%%%%%%%%%%%%%%

%%%%%%%%%%%%%%%%%%%%%%%%%%%%%%%%%%%%%%%%
\subsection{Training and Evaluation Configuration}
\label{sec:training_config}
%%%%%%%%%%%%%%%%%%%%%%%%%%%%%%%%%%%%%%%%

%%%%%%%%%%%%%%%%%%%%%%%%%%%%%%%
\begin{table}[t]
\centering
\small
\caption{Detailed model configurations used in our experiments, including parameter count, per-forward-pass FLOPs, training resolution, learning rate, guidance scale, and inference steps. $^\dagger$ denotes models additionally evaluated at $1024\!\times\!1024$ resolution.}
\label{tab:config}
\setlength{\tabcolsep}{2.5mm}
\renewcommand{\arraystretch}{1.1}
\begin{tabular}{l|c|c|c|c|c|c}
\toprule
\textbf{Model} & \textbf{Params (B)} & \textbf{FLOPs (G)} & \textbf{Resolution} & \textbf{LR} & \textbf{Guidance} & \textbf{Infer. Steps} \\
\midrule
\textbf{SANA (0.6B)}             & 0.6 & 548  & $512\times512$   & 1e-4 & 5   & 20 \\
\textbf{SANA (1.6B)}             & 1.6 & 567  & $512\times512$   & 1e-4 & 5   & 20 \\
\textbf{PixArt-$\Sigma$}        & 0.6 & 567  & $512\times512$   & 2e-5 & 5 & 20 \\
\midrule
\textbf{SANA (0.6B)$^\dagger$ }  & 0.6 & 1397 & $1024\times1024$ & 1e-4 & 5   & 20 \\
\textbf{SANA (1.6B)$^\dagger$ }  & 1.6 & 2240 & $1024\times1024$ & 1e-4 & 5   & 20 \\
\textbf{PixArt-$\Sigma^\dagger$ }& 0.6 & 1399 & $1024\times1024$ & 2e-5 & 5 & 20 \\
\textbf{FLUX.2 klein base }      & 4.0 & 6953 & $1024\times1024$ & 1e-5 & 4   & 50 \\
\bottomrule
\end{tabular}
\end{table}
\begin{figure*}[h]
  \centering
  \includegraphics[width=\textwidth]{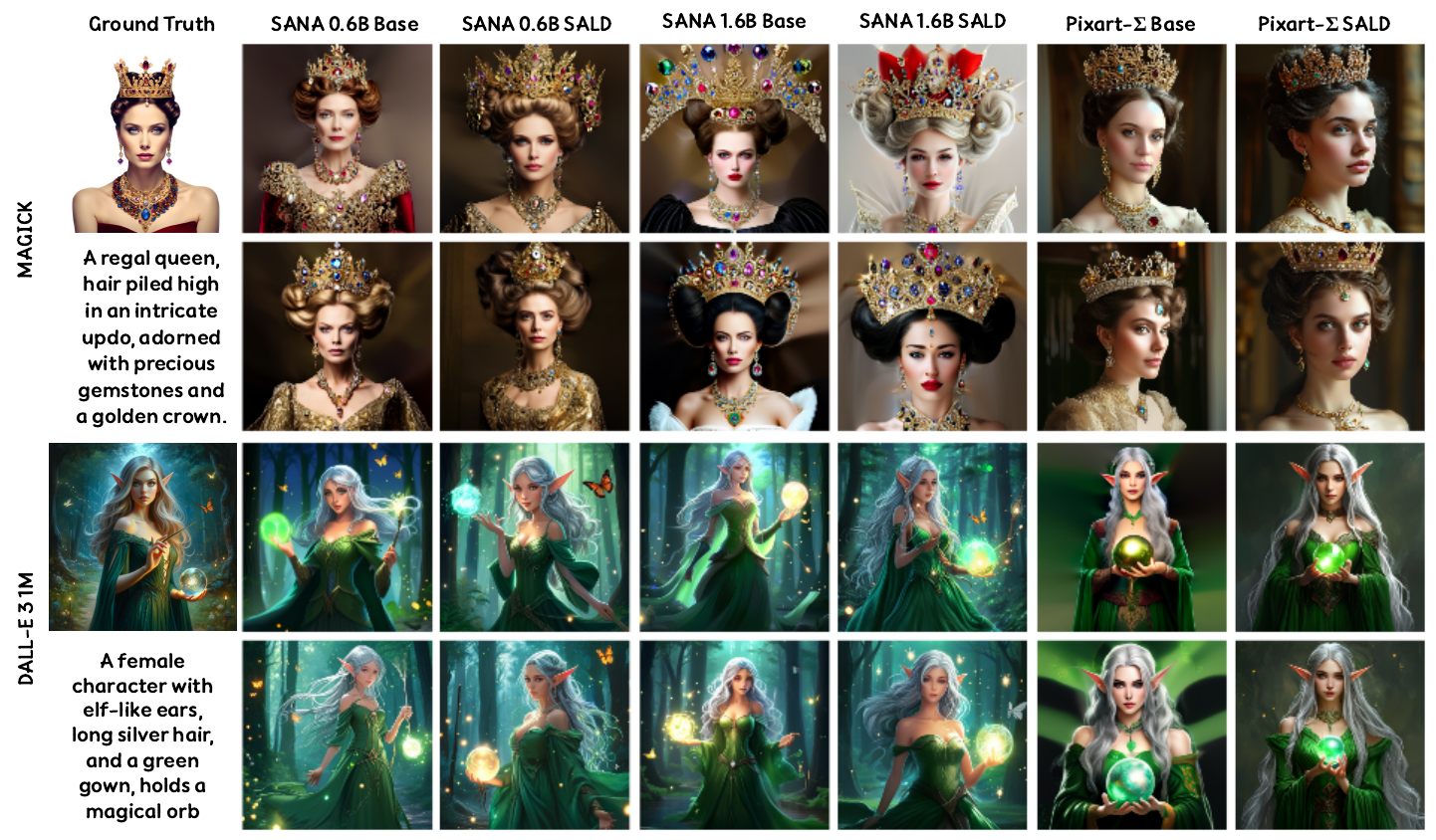}
  \caption{Qualitative comparison across baseline models and SALD on 1024 × 1024 datasets.}
  \label{fig:random3}
  \vspace{-5mm}
\end{figure*}
%%%%%%%%%%%%%%%%%%%%%%%%%%%%%%%
\begin{figure}[t]
  \centering
  \includegraphics[width=\textwidth]{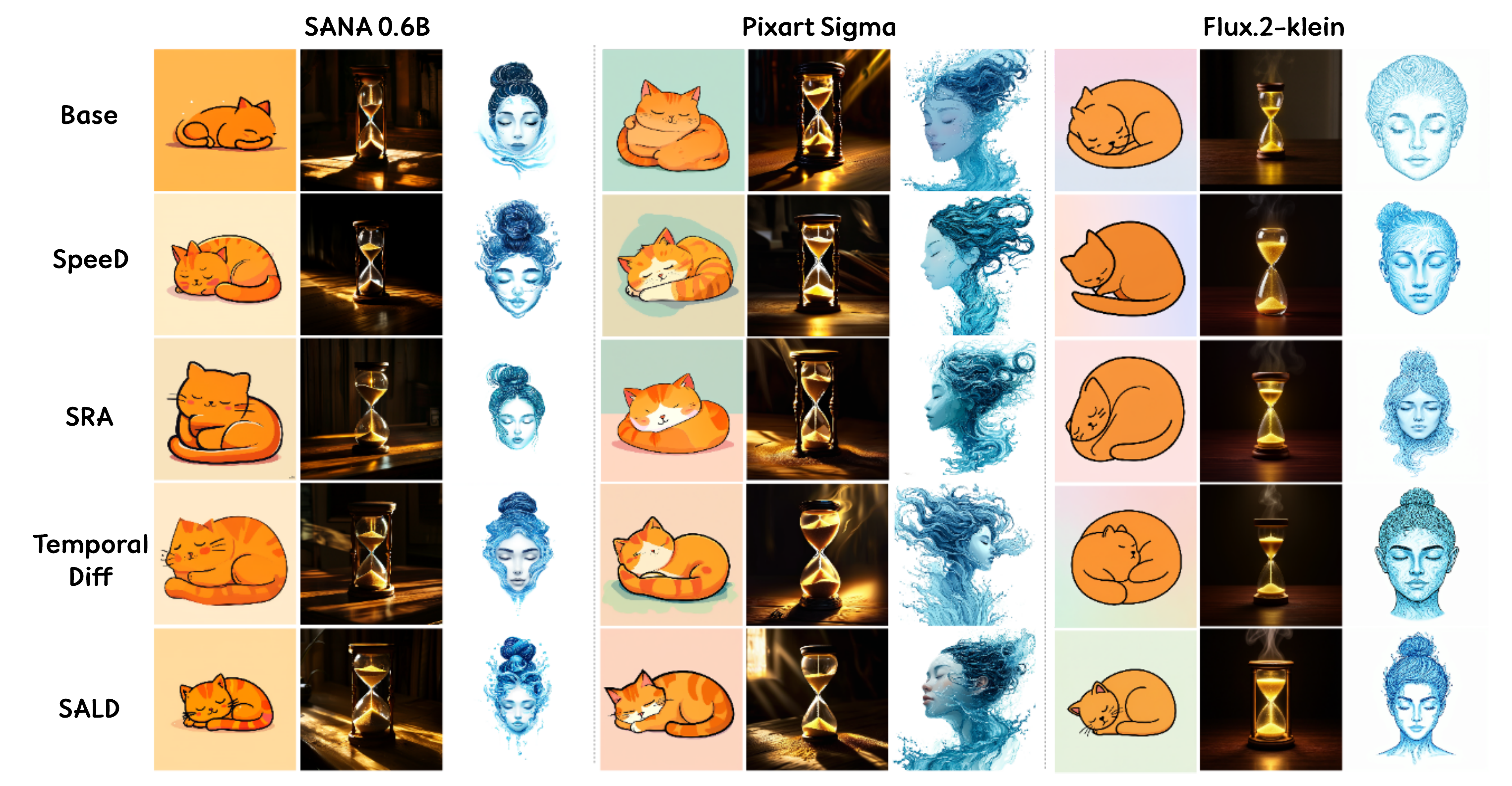}
  \caption{Qualitative comparisons: samples generated by SALD and other different Baselines on 1024 × 1024 Datasets. Captions are, Column-1: A cute orange cat sleeping curled up, flat vector illustration style. Column-2: An hourglass with golden sand, dramatic cinematic lighting on a wooden table. Column-3: A woman's face in profile with flowing blue hair made of water, digital art.}
  \label{fig:baseline_comp_1024}
\end{figure}
%%%%%%%%%%%%%%%%%%%%%%%%%%%%%%%%%%%%%%%%%%%%%%%%%%%%
%%%%%%%%%%%%%%%%%%%%%%%%%%%%%%%
\begin{figure*}[t]
  \centering
  \includegraphics[width=\textwidth]{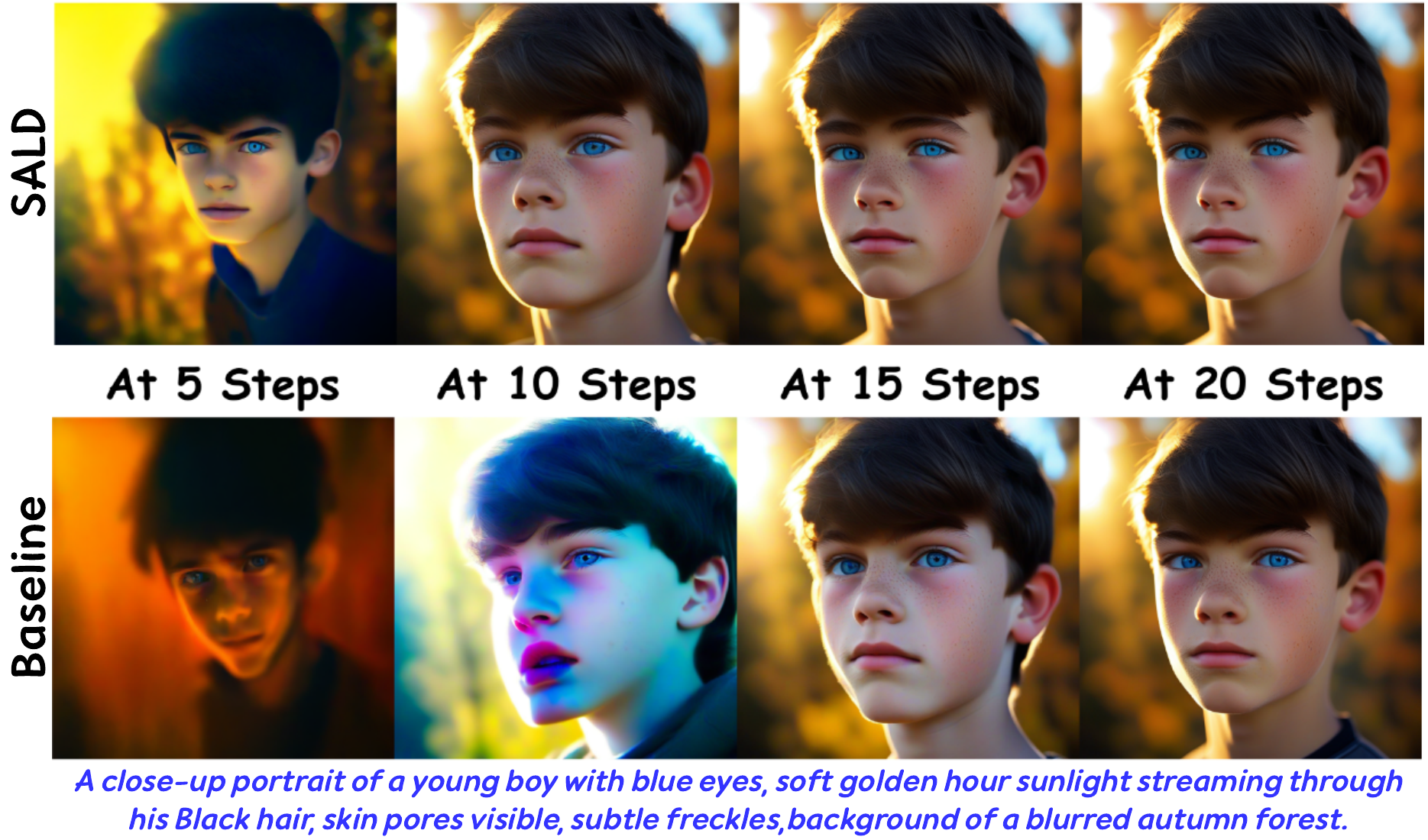}
  \caption{Qualitative comparison of inference-time generation across different sampling steps for SALD and the baseline using SANA 1.6B trained on DALL$\cdot$E~3~1M.}
  \label{fig:c1}
  \vspace{-5mm}
\end{figure*}
%%%%%%%%%%%%%%%%%%%%%%%%%%%%%%%

%%%%%%%%%%%%%%%%%%%%%%%%%%%%%%%
\begin{table}[t]
\vspace{-2mm}
% \tiny
\caption{Wall-clock time and quantitative comparison on CUB-200 and MultiGen datasets. Best values in \textcolor{blue}{Blue}}
\vspace{-2mm}
\label{tab:time_result}
\centering
\scriptsize
\resizebox{0.9\linewidth}{!}{%
\begin{tabular}{c l c c c | c c c}
\toprule
\textbf{Dataset} & \textbf{Model} & \textbf{Variant} & \textbf{Params (B)} & \textbf{Steps} & \textbf{FID}$\downarrow$ & \textbf{CLIP}$\uparrow$ & \textbf{LPIPS}$\downarrow$ \\
\midrule

% -------- CUB-200 --------
\multirow{15}{*}{\rotatebox{90}{\textbf{CUB-200}}}
& \texttt{\textbf{SANA}} & Base & 1.6 & 12.5k & 21.47 & 26.01 & 0.7709 \\
& & SpeeD & 1.6 & 12.5k & 20.83 & 26.08 & 0.7670 \\
& & SRA & 1.6 & 12.5k & 20.72 & 26.15 & 0.7632 \\
& & Temporal Diff & 1.6 & 12.5k & 20.61 & 26.22 & 0.7593 \\
& & SALD & 1.6 & \textbf{10k} & \textcolor{blue}{\textbf{18.16}} & \textcolor{blue}{\textbf{26.29}} & \textcolor{blue}{\textbf{0.7537}} \\
\addlinespace
& \texttt{\textbf{SANA}} & Base & 0.6 & 12.5k & 25.54 & 26.19 & 0.7802 \\
& & SpeeD & 0.6 & 12.5k & 24.77 & 26.26 & 0.7763 \\
& & SRA & 0.6 & 12.5k & 24.65 & 26.33 & 0.7724 \\
& & Temporal Diff & 0.6 & 12.5k & 24.52 & 26.40 & 0.7705 \\
& & SALD & 0.6 & \textbf{10k} & \textcolor{blue}{\textbf{23.70}} & \textcolor{blue}{\textbf{26.61}} & \textcolor{blue}{\textbf{0.7686}} \\
\addlinespace
& \texttt{\textbf{PixArt-$\Sigma$}} & Base & 0.6 & 12.5k & 21.57 & 29.27 & 0.7886 \\
& & SpeeD & 0.6 & 12.5k & 20.92 & 29.35 & 0.7847 \\
& & SRA & 0.6 & 12.5k & 20.81 & 29.42 & 0.7808 \\
& & Temporal Diff & 0.6 & 12.5k & 20.70 & 29.50 & 0.7768 \\
& & SALD & 0.6 & \textbf{10k} & \textcolor{blue}{\textbf{20.50}} & \textcolor{blue}{\textbf{29.55}} & \textcolor{blue}{\textbf{0.7452}} \\

\midrule

% -------- MultiGen --------
\multirow{15}{*}{\rotatebox{90}{\textbf{MultiGen}}}
& \texttt{\textbf{SANA}} & Base & 1.6 & 12.5k & 25.87 & 26.26 & 0.7593 \\
& & SpeeD & 1.6 & 12.5k & 25.09 & 26.33 & 0.7555 \\
& & SRA & 1.6 & 12.5k & 24.96 & 26.40 & 0.7517 \\
& & Temporal Diff & 1.6 & 12.5k & 24.84 & 26.47 & 0.7550 \\
& & SALD & 1.6 & \textbf{10k} & \textcolor{blue}{\textbf{23.86}} & \textcolor{blue}{\textbf{26.29}} & \textcolor{blue}{\textbf{0.7533}} \\
\addlinespace
& \texttt{\textbf{SANA}} & Base & 0.6 & 12.5k & 31.89 & 26.68 & 0.7648 \\
& & SpeeD & 0.6 & 12.5k & 30.93 & 26.75 & 0.7610 \\
& & SRA & 0.6 & 12.5k & 30.77 & 26.82 & 0.7572 \\
& & Temporal Diff & 0.6 & 12.5k & 30.61 & 26.89 & 0.7534 \\
& & SALD & 0.6 & \textbf{10k} & \textcolor{blue}{\textbf{29.15}} & \textcolor{blue}{\textbf{26.69}} & \textcolor{blue}{\textbf{0.7472}} \\
\addlinespace
& \texttt{\textbf{PixArt-$\Sigma$}} & Base & 0.6 & 12.5k & 31.57 & 25.88 & 0.7702 \\
& & SpeeD & 0.6 & 12.5k & 30.62 & 25.95 & 0.7664 \\
& & SRA & 0.6 & 12.5k & 30.46 & 26.02 & 0.7626 \\
& & Temporal Diff & 0.6 & 12.5k & 30.31 & 26.09 & 0.7588 \\
& & SALD & 0.6 & \textbf{10k} & \textcolor{blue}{\textbf{30.47}} & \textcolor{blue}{\textbf{25.93}} & \textcolor{blue}{\textbf{0.7292}} \\

\bottomrule
\end{tabular}%
}
% \vspace{-5mm}
\end{table}
%%%%%%%%%%%%%%%%%%%%%%%%%%%%%%%

%%%%%%%%%%%%%%%%%%%%%%%%%%%%%%%
\begin{table}[t]
\tiny
\caption{
Quantitative comparison across different inference steps on the Oxford Flowers dataset. Light blue highlights indicate the metric values for the \textbf{SALD} variant.
}
\vspace{-2mm}
\label{tab:steps_comparison}
\centering
\small
\resizebox{\linewidth}{!}{%
\begin{tabular}{l c c | c c c | c c c | c c c}
\toprule
\textbf{Model} & \textbf{Params (B)} & \textbf{Variant} 
& \multicolumn{3}{c|}{\textbf{Steps = 10}} 
& \multicolumn{3}{c|}{\textbf{Steps = 15}} 
& \multicolumn{3}{c}{\textbf{Steps = 20}} \\
\cmidrule(lr){1-3} \cmidrule(lr){4-6} \cmidrule(lr){7-9} \cmidrule(lr){10-12}
& & & \textbf{FID}$\downarrow$ & \textbf{CLIP}$\uparrow$ & \textbf{LPIPS}$\downarrow$ 
& \textbf{FID}$\downarrow$ & \textbf{CLIP}$\uparrow$ & \textbf{LPIPS}$\downarrow$ 
& \textbf{FID}$\downarrow$ & \textbf{CLIP}$\uparrow$ & \textbf{LPIPS}$\downarrow$ \\
\midrule

% -------- Oxford Flowers --------
\texttt{\textbf{SANA}} & 1.6 & Base & 55.06 & 25.35 & 0.8153 & 37.49 & 25.60 & 0.8055 & 33.11 & 26.35 & 0.7896 \\
\rowcolor{blue!15}
\cellcolor{white} & \cellcolor{white} & \cellcolor{white}\textbf{SALD} & 41.56 & 25.48 & 0.8088 & 30.60 & 25.72 & 0.7962 & 27.37 & 26.57 & 0.7781 \\
\addlinespace
\texttt{\textbf{SANA}} & 0.6 & Base & 44.05 & 24.31 & 0.8059 & 35.26 & 25.29 & 0.7971 & 30.64 & 26.43 & 0.7805 \\
\rowcolor{blue!15}
\cellcolor{white} & \cellcolor{white} & \cellcolor{white}\textbf{SALD} & 33.68 & 24.48 & 0.7906 & 25.54 & 25.63 & 0.7865 & 25.83 & 26.65 & 0.7728 \\
\addlinespace
\texttt{\textbf{PixArt-$\Sigma$}} & 0.6 & Base & 49.12 & 25.10 & 0.8120 & 43.25 & 25.25 & 0.7910 & 37.88 & 25.37 & 0.7687 \\
\rowcolor{blue!15}
\cellcolor{white} & \cellcolor{white} & \cellcolor{white}\textbf{SALD} & 46.35 & 25.28 & 0.7985 & 40.80 & 25.38 & 0.7820 & 36.48 & 25.42 & 0.7682 \\

\bottomrule
\end{tabular}}%

\vspace{-3mm}
\end{table}
%%%%%%%%%%%%%%%%%%%%%%%%%%%%%%%

%%%%%%%%%%%%%%%%%%%%%%%%%%%%%%%%%%%%%%%%%%%%%%%%%%%%%%%%%%%%%%%%%

Table~\ref{tab:config} summarizes the implementation details of all backbone architectures used in our experiments. We evaluate SALD on SANA, PixArt-$\Sigma$, and FLUX.2 variants across multiple parameter scales and image resolutions to demonstrate the generality of the proposed framework. The table reports the number of trainable parameters, computational complexity measured in FLOPs per forward pass, training resolution, the learning rate used for finetuning, as well as the inference-time guidance scale and number of sampling steps. Experiments are conducted at both $512\times512$ and $1024\times1024$ resolutions, where higher resolutions naturally require substantially higher computational cost due to increased latent spatial dimensions. All models are trained using bfloat16 (bf16) mixed-precision training with gradient clipping set to a maximum norm of $1.0$. Following the original training configurations of each architecture, we employ a constant with warmup learning rate scheduler with $500$ warmup steps, after which the learning rate is kept constant. We use relatively larger learning rates for SANA models and more conservative rates for PixArt variants to ensure stable optimization. These settings are kept consistent across baseline and SALD experiments for fair comparison.
For evaluation, we use the following implementations: FID is computed using \href{https://github.com/Lightning-AI/torchmetrics/tree/master}{TorchMetrics}, CLIP-based similarity is evaluated using \href{https://huggingface.co/openai/clip-vit-large-patch14}{CLIP ViT-L/14}, and LPIPS is computed using the \href{https://github.com/richzhang/perceptualsimilarity}{Perceptual Similarity (LPIPS)} implementation.

%%%%%%%%%%%%%%%%%%%%%%%%%%%%%%%%%%%%%%%%%%%%%%%%%%
\subsection{Wall-Clock Evaluation}
We further evaluate the training efficiency of SALD under a matched wall-clock budget. As shown in Table~\ref{tab:time_result}, SALD incurs approximately $25\%$ additional wall-clock time per training step compared with the corresponding baseline, measured on the same NVIDIA RTX A6000 setup. This overhead comes from the additional easy noise reference path: SALD performs gradient computation and backpropagation only for the hard noise path, while the easy noise path is evaluated without gradient computation, requiring only an additional forward pass. To account for this overhead, SALD is trained for $10$k steps, while the baselines are trained for $12.5$k steps. Thus, the total training budgets are approximately matched, since $10\text{k}\times1.25\approx12.5\text{k}$. As reported in Table~\ref{tab:time_result}, under this comparable wall-clock budget, SALD consistently achieves better generation quality across both CUB-200 and MultiGen datasets. In particular, SALD improves FID, CLIP, and LPIPS over the corresponding baseline models across all evaluated backbone configurations. These results show that the limited additional cost introduced by the no-gradient reference path is compensated by SALD's faster convergence, allowing it to achieve superior generation quality within approximately the same overall training time as the baselines.

%%%%%%%%%%%%%%%%%%%%%%%%%%%%%%%%%%%%%%%%%%%%%%
\subsection{Impact of Inference Steps and Model Scalability}

Table~\ref{tab:steps_comparison} reveals a consistent sampling-efficiency advantage of SALD across model scales, from SANA 0.6B/1.6B to PixArt-$\Sigma$ 0.6B, on Oxford Flowers. As the number of inference steps is reduced, SALD preserves generation quality more effectively than the corresponding Base models, yielding lower FID and LPIPS while retaining comparable or higher CLIP scores. In particular, for both SANA models, SALD at 15 steps attains lower FID than the Base model at 20 steps (30.60 vs.\ 33.11 for SANA 1.6B and 25.54 vs.\ 30.64 for SANA 0.6B), using 25\% fewer sampling steps, although LPIPS and CLIP at 15 steps remain slightly below the 20-step Base values.

The behavior is consistent with the relative error signal introduced in Eq.~\ref{eq:advantage}. Instead of treating prediction errors at different noise states independently, SALD contrasts the hard and easy states of the same sample and assigns greater importance when the harder state exhibits a larger error. Through the AGD weighting in Eq.~\ref{eq:agd-weight}, optimization therefore concentrates on trajectory regions where prediction remains comparatively difficult. This encourages more accurate learning of the velocity field along challenging portions of the noise trajectory, which in turn reduces the dependence on densely spaced inference updates. The cross-scale results in Table~\ref{tab:steps_comparison} suggest that this effect is not specific to a particular architecture size, but translates into consistently stronger quality--efficiency trade-offs across the evaluated models.

Furthermore, Figure~\ref{fig:c1} provides a qualitative view of this faster inference-time convergence using SANA 1.6B trained on DALL$\cdot$E~3~1M. The distinction is most pronounced at low sampling budgets. SALD establishes the correct facial structure and overall appearance within the first few steps, with the 10-step output already closely resembling its 15- and 20-step generations. In contrast, the baseline remains highly unstable at 5 and 10 steps, exhibiting substantial changes in geometry, illumination, and color before converging toward a consistent image at later steps. Thus, additional sampling primarily refines the SALD output rather than correcting its global structure, whereas the baseline requires more iterations to stabilize both structure and appearance. This visual progression further supports that SALD learns a more efficient generation trajectory, allowing coherent samples to emerge under substantially fewer inference steps.

%%%%%%%%%%%%%%%%%%%%%%%%%%%%%%%%%%%%%%%%%%%%%%
\begin{table}[htbp]
    \centering
    \caption{Evaluation metrics across training steps for SANA 1.6 and SANA 0.6 on MultiGen and Oxford 102 Flowers datasets. Lower FID and LPIPS ($\downarrow$) are better, while higher CLIP ($\uparrow$) is better.}
    \label{tab:eval_sana_datasets}
    \vspace{-2mm}
    \renewcommand{\arraystretch}{1.1}
    \resizebox{\textwidth}{!}{%
    \begin{tabular}{c c | ccc | ccc | ccc | ccc}
        \toprule
        
        \multirow{3}{*}{\textbf{Steps}} 
        & \multirow{3}{*}{\textbf{Variant}} 
        & \multicolumn{6}{c|}{\textbf{MultiGen}} 
        & \multicolumn{6}{c}{\textbf{Oxford 102 Flowers}} \\
        \cmidrule(lr){3-8} \cmidrule(lr){9-14}
        
        & & \multicolumn{3}{c|}{\textbf{SANA 1.6}} 
        & \multicolumn{3}{c|}{\textbf{SANA 0.6}} 
        & \multicolumn{3}{c|}{\textbf{SANA 1.6}} 
        & \multicolumn{3}{c}{\textbf{SANA 0.6}} \\
        \cmidrule(lr){3-5} \cmidrule(lr){6-8} \cmidrule(lr){9-11} \cmidrule(lr){12-14}
        
        & & FID & CLIP & LPIPS 
        & FID & CLIP & LPIPS 
        & FID & CLIP & LPIPS 
        & FID & CLIP & LPIPS \\
        \midrule
        
        \multirow{2}{*}{10k}  & Baseline & 26.44 & 26.25 & 0.7541 & 31.93 & 26.68 & 0.7495 & 33.11 & 26.35 & 0.7896 & 30.64 & 26.43 & 0.7805 \\
                              & SALD     & 23.86 & 26.29 & 0.7533 & 29.15 & 26.69 & 0.7472 & 27.37 & 26.57 & 0.7781 & 25.83 & 26.65 & 0.7728 \\
        \midrule
        \multirow{2}{*}{20k}  & Baseline & 26.18 & 26.31 & 0.7550 & 31.42 & 26.48 & 0.7605 & 32.45 & 26.30 & 0.7842 & 30.41 & 26.38 & 0.7761 \\
                              & SALD     & 23.52 & 26.35 & 0.7542 & 28.78 & 26.52 & 0.7598 & 26.80 & 26.52 & 0.7745 & 25.70 & 26.60 & 0.7695 \\
        \midrule
        \multirow{2}{*}{30k}  & Baseline & 25.90 & 26.19 & 0.7548 & 31.10 & 26.47 & 0.7588 & 31.82 & 26.24 & 0.7791 & 30.19 & 26.32 & 0.7719 \\
                              & SALD     & 23.21 & 26.38 & 0.7540 & 28.45 & 26.50 & 0.7580 & 26.35 & 26.46 & 0.7712 & 25.58 & 26.54 & 0.7665 \\
        \midrule
        \multirow{2}{*}{40k}  & Baseline & 25.65 & 26.15 & 0.7539 & 30.87 & 26.43 & 0.7588 & 31.25 & 26.18 & 0.7743 & 29.98 & 26.26 & 0.7680 \\
                              & SALD     & 22.95 & 26.40 & 0.7535 & 28.15 & 26.48 & 0.7575 & 25.98 & 26.40 & 0.7682 & 25.47 & 26.48 & 0.7638 \\
        \midrule
        \multirow{2}{*}{50k}  & Baseline & 25.47 & 26.08 & 0.7537 & 30.29 & 26.39 & 0.7596 & 30.75 & 26.12 & 0.7698 & 29.79 & 26.20 & 0.7643 \\
                              & SALD     & 22.73 & 26.41 & 0.7532 & 28.63 & 26.41 & 0.7576 & 25.68 & 26.34 & 0.7655 & 25.38 & 26.41 & 0.7613 \\
        \midrule
        \multirow{2}{*}{60k}  & Baseline & 25.24 & 26.11 & 0.7518 & 30.75 & 26.44 & 0.7584 & 30.30 & 26.06 & 0.7656 & 29.61 & 26.14 & 0.7609 \\
                              & SALD     & 22.55 & 26.42 & 0.7520 & 28.35 & 26.45 & 0.7570 & 25.42 & 26.27 & 0.7630 & 25.30 & 26.34 & 0.7590 \\
        \midrule
        \multirow{2}{*}{70k}  & Baseline & 25.08 & 26.11 & 0.7520 & 30.86 & 26.44 & 0.7586 & 29.90 & 26.00 & 0.7617 & 29.45 & 26.08 & 0.7577 \\
                              & SALD     & 22.40 & 26.42 & 0.7518 & 28.18 & 26.46 & 0.7568 & 25.20 & 26.20 & 0.7607 & 25.23 & 26.27 & 0.7569 \\
        \midrule
        \multirow{2}{*}{80k}  & Baseline & 24.94 & 26.11 & 0.7522 & 30.02 & 26.42 & 0.7587 & 29.55 & 25.94 & 0.7581 & 29.30 & 26.02 & 0.7547 \\
                              & SALD     & 22.28 & 26.42 & 0.7517 & 28.05 & 26.45 & 0.7565 & 25.02 & 26.12 & 0.7586 & 25.18 & 26.19 & 0.7550 \\
        \midrule
        \multirow{2}{*}{90k}  & Baseline & 24.82 & 26.11 & 0.7521 & 30.02 & 26.43 & 0.7586 & 29.25 & 25.89 & 0.7548 & 29.16 & 25.96 & 0.7519 \\
                              & SALD     & 22.18 & 26.42 & 0.7516 & 27.95 & 26.44 & 0.7562 & 24.95 & 26.04 & 0.7567 & 25.15 & 26.11 & 0.7533 \\
        \midrule
        \multirow{2}{*}{100k} & Baseline & 24.70 & 26.11 & 0.7522 & 30.00 & 26.43 & 0.7586 & 28.61 & 25.84 & 0.7378 & 28.71 & 25.84 & 0.7373 \\
                              & SALD     & 22.12 & 26.21 & 0.7531 & 28.04 & 26.39 & 0.7539 & 24.90 & 25.75 & 0.7393 & 25.14 & 25.93 & 0.7335 \\
        \bottomrule
    \end{tabular}%
    }
    \vspace{-2mm}
\end{table}

%%%%%%%%%%%%%%%%%%%%%%%%%%%%%%%
\subsection{What happens if we train 10$\times$ More?}
\label{appen:convergence}

To examine whether the gains reported at the main training budget are a consequence of insufficient training of the Baseline, we extend both the Baseline and SALD models to $100$k training steps, corresponding to $10\times$ the standard $10$k-step training budget. Table~\ref{tab:eval_sana_datasets} reports the results across SANA 1.6B and SANA 0.6B on the MultiGen and Oxford 102 Flowers datasets. SALD maintains a consistent FID advantage over the corresponding Baseline throughout the extended training horizon, while CLIP and LPIPS remain comparable. For example, at $10$k steps, SALD already achieves substantially lower FID than the Baseline, reducing it from $26.44$ to $23.86$ on MultiGen and from $33.11$ to $27.37$ on Oxford 102 Flowers with SANA 1.6B. Importantly, this advantage does not disappear as training continues, demonstrating that the improvements observed at the main $10$k-step setting are not simply a consequence of under-training the Baseline.

The extended training results also reveal diminishing returns at later stages of optimization. After approximately $80$k steps, both models improve more slowly. SALD maintains a clear FID advantage over the Baseline in this later regime: the gap remains stable on MultiGen ($2.58$ at both $10$k and $100$k for SANA 1.6B) and narrows but persists on Oxford 102 Flowers (from $5.74$ to $3.71$ for SANA 1.6B and from $4.81$ to $3.57$ for SANA 0.6B). Thus, extending training to $100$k steps does not eliminate the advantage of SALD. Instead, the results demonstrate that the gain introduced by SALD persists throughout the extended training process, while further optimization beyond $80$k steps provides only marginal improvements for either method. These observations support the use of $10$k steps as a practical training budget for the main experiments and confirm that the reported improvements are not due to premature termination of Baseline training.

%%%%%%%%%%%%%%%%%%%%%%%%%%%%%%%%%%%%%%%%%%%%%%%
\subsection{Additional Visual Samples}
%%%%%%%%%%%%%%%%%%%%%%%%%%%%%%%

We provide additional visual samples generated by SALD across multiple datasets, architectures, and resolutions to analyze its ability to preserve semantic consistency,  structural fidelity, and fine-grained visual details under diverse generation settings.

Figure~\ref{fig:random3} presents qualitative comparisons between
baseline finetuning and SALD across multiple architectures on the MAGICK and
DALL$\cdot$E~3~1M datasets at $1024\times1024$ resolution. Across all
settings, SALD produces sharper and more structurally coherent generations
with improved preservation of fine-grained semantic details. On MAGICK,
baseline models often generate unstable crown geometry, distorted facial
structure, inconsistent jewelry placement, and blurred ornamental details,
whereas SALD maintains cleaner symmetry, sharper textures, and more realistic
facial composition. Similarly, on DALL$\cdot$E~3~1M, baseline generations
exhibit inconsistent hand--object interactions, oversmoothed hair textures,
and weaker illumination consistency around the magical orb. In contrast,
SALD generates more stable anatomy, finer hair strands, improved lighting
coherence, and stronger object-boundary preservation. These improvements are
consistent across SANA and PixArt variants, demonstrating that SALD
effectively enhances high-frequency detail generation and semantic
consistency at high resolutions.

Figure~\ref{fig:baseline_comp_1024} further showcases diverse samples generated by SALD using
PixArt-$\Sigma$ at $1024\times1024$ resolution on the DALL$\cdot$E~3~1M
dataset. Additionally, Figure~\ref{fig:q3} further showcases diverse samples generated by SALD using Flux.2 klein at $1024\times1024$ resolution on the DALL$\cdot$E~3~1M
dataset. The generated images demonstrate strong semantic consistency,
realistic illumination, and preservation of fine-grained textures across a
wide variety of scenes and subjects. In portrait generations, SALD maintains
detailed facial structure, natural skin textures, and coherent lighting
interactions. For reflective and transparent objects, such as the crystal
sphere, the model accurately captures highlights, reflections, and depth
cues. Indoor and environmental scenes exhibit consistent global composition
and soft lighting transitions, while animal generations preserve
high-frequency fur textures and sharp eye details. These results highlight
the ability of SALD to generate visually coherent and high-fidelity images
across diverse semantic categories at high resolution.

Figure~\ref{fig:q2} presents qualitative generations produced by SALD across
diverse datasets, resolutions, and backbone architectures, demonstrating both
visual fidelity and semantic consistency. In the first row, SALD applied to
SANA 1.6B trained on DALL$\cdot$E~3~1M generates highly detailed urban neon
scenes with realistic illumination effects, accurate text rendering, and
strong spatial coherence at $1024\times1024$ resolution. The reflections on
wet cobblestone surfaces and the natural bokeh distribution highlight the
model's ability to preserve fine-grained lighting structure and texture
realism. The second row showcases generations on the CUB-200 dataset, where
SALD captures intricate feather structures, wing symmetry, and motion-aware
details in a snowy owl generated from an out-of-dataset prompt, while maintaining sharp focus on salient regions
such as the eyes and plumage patterns. Finally, the third row demonstrates
PixArt-$\Sigma$ generations on Oxford-102 Flowers, where SALD produces smooth
petal geometry, subtle color gradients, and realistic depth-of-field effects
with visually pleasing softness and fine texture retention. Across all
settings, the generated samples exhibit strong compositional stability,
high-frequency detail preservation, and improved prompt alignment, indicating
that SALD generalizes effectively across both natural and artistic image
domains.
%%%%%%%%%%%%%%%%%%%%%%%%%%%%%%%%%%%%%%%%%%%%%%%%%%%%%%%%%%%%%%%%%%%

%%%%%%%%%%%%%%%%%%%%%%%%%%%%%%%%
\vspace{-2mm}
\section{Limitations}\label{sec:limitations}
\vspace{-2mm}

SALD evaluates each training example through two noise levels using the same diffusion model: the lower-noise easy path provides an internal reference under \texttt{no\_grad}, whereas the higher-noise hard path supplies the training gradient. Consequently, SALD requires neither an external teacher nor additional trainable model parameters during training or inference. Nevertheless, the additional easy-path evaluation increases the per-step training cost by approximately $25\%$. Although our wall-clock-matched experiments (Table~\ref{tab:time_result}) show that the resulting improvement in optimization efficiency can offset this overhead, reducing the per-step training cost remains an important direction for future work. One possibility is to cache easy-path reference predictions from earlier checkpoints. However, such cached predictions would become stale as the model changes and would depart from the formulation considered in this work, where the reference is always computed with the current model parameters.

TAM stores one non-trainable scalar for each training sample. This memory cost is negligible for the dataset sizes considered in our experiments but grows linearly with the number of training examples. Moreover, TAM updates a sample's memory only when that sample is revisited, so its curriculum effect is weaker when each example is seen only a few times during training. For extremely large web-scale datasets, group-level, clustered, or compressed difficulty statistics could reduce this cost and increase the update frequency, although such approximations may sacrifice the precision of the per-sample curriculum.

Finally, our empirical evaluation is limited to fine-tuning pretrained text-to-image flow-matching models; we do not evaluate SALD for training from scratch. Our convergence analysis also holds only for a fixed TAM memory state (Remark~4). While we describe how the same easy--hard self-referencing principle can be extended to DDPM-style noise-prediction objectives, we do not experimentally validate that extension in this work. Evaluating SALD with noise-prediction models, in from-scratch training, and in other generative modalities, including video and 3D generation, remain directions for future research.

%%%%%%%%%%%%%%%%%%%%%%%%%%%%%%%%%%%%%%%%%%%%%%%%%%%%%%%%%%

%%%%%%%%%%%%%%%%%%%%%%%%%%%%%%%
\begin{figure}[H]
  \vspace{-2mm}
  \centering
  \includegraphics[
    width=0.99\textwidth,
    trim=0 10mm 0 10mm,
    clip
]{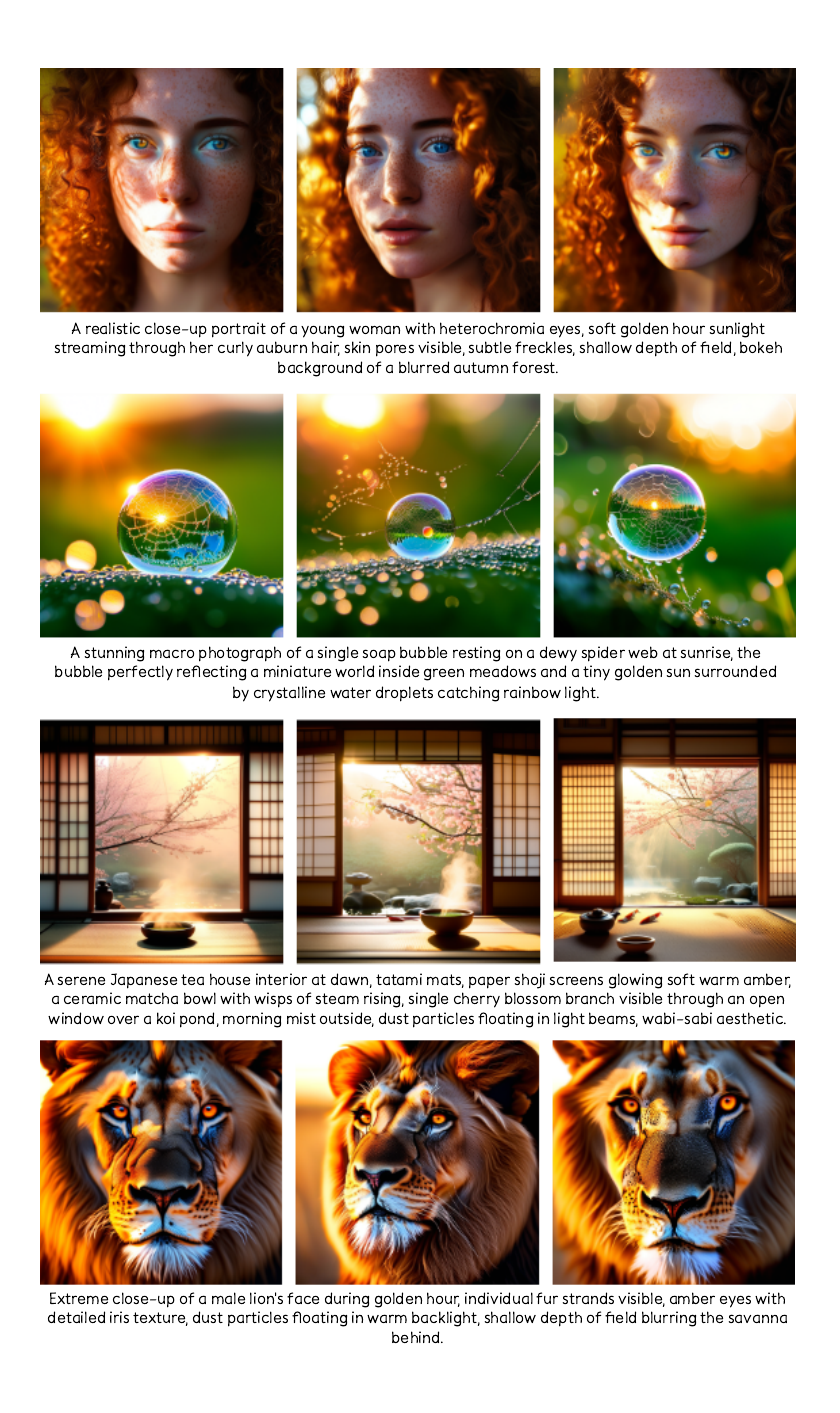}

  % \vspace{-10mm}
  \caption{Qualitative samples generated by SALD using PixArt-$\Sigma$ at $1024\!\times\!1024$ resolution on the DALL$\cdot$E~3~1M dataset.}
  \label{fig:q1}
\end{figure}
%%%%%%%%%%%%%%%%%%%%%%%%%%%%%%%
\begin{figure}[H]
  \centering
  \includegraphics[
    width=0.99\textwidth,
    trim=0 10mm 0 10mm,
    clip
]{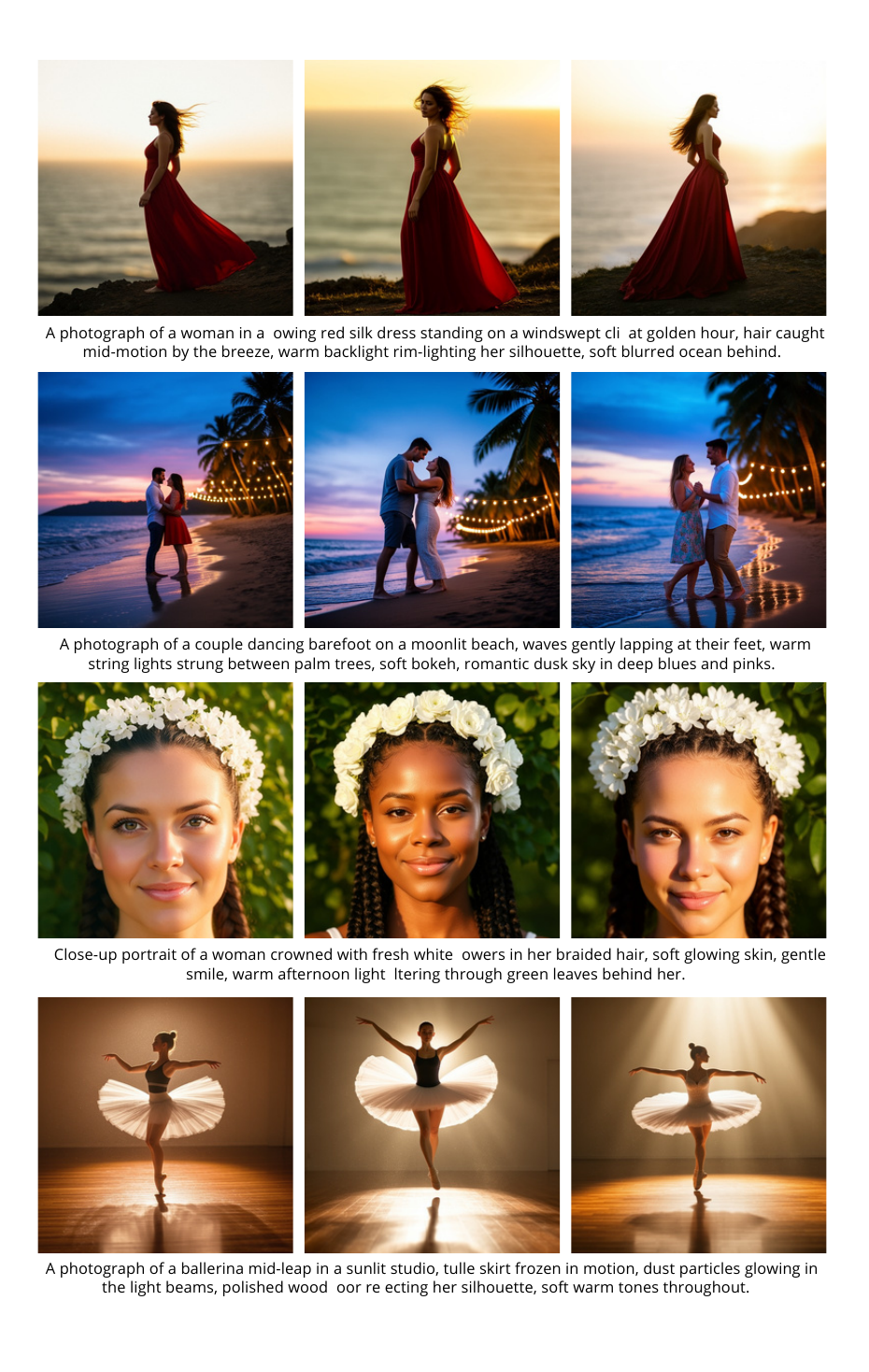}
  \caption{Qualitative samples generated by SALD using Flux.2 klein Base at $1024\!\times\!1024$ resolution on the DALL$\cdot$E~3~1M dataset.}
  \label{fig:q3}
\end{figure}

%%%%%%%%%%%%%%%%%%%%%%%%%%%%%%%
\begin{figure}[H]
    \centering
    \includegraphics[width=\textwidth]{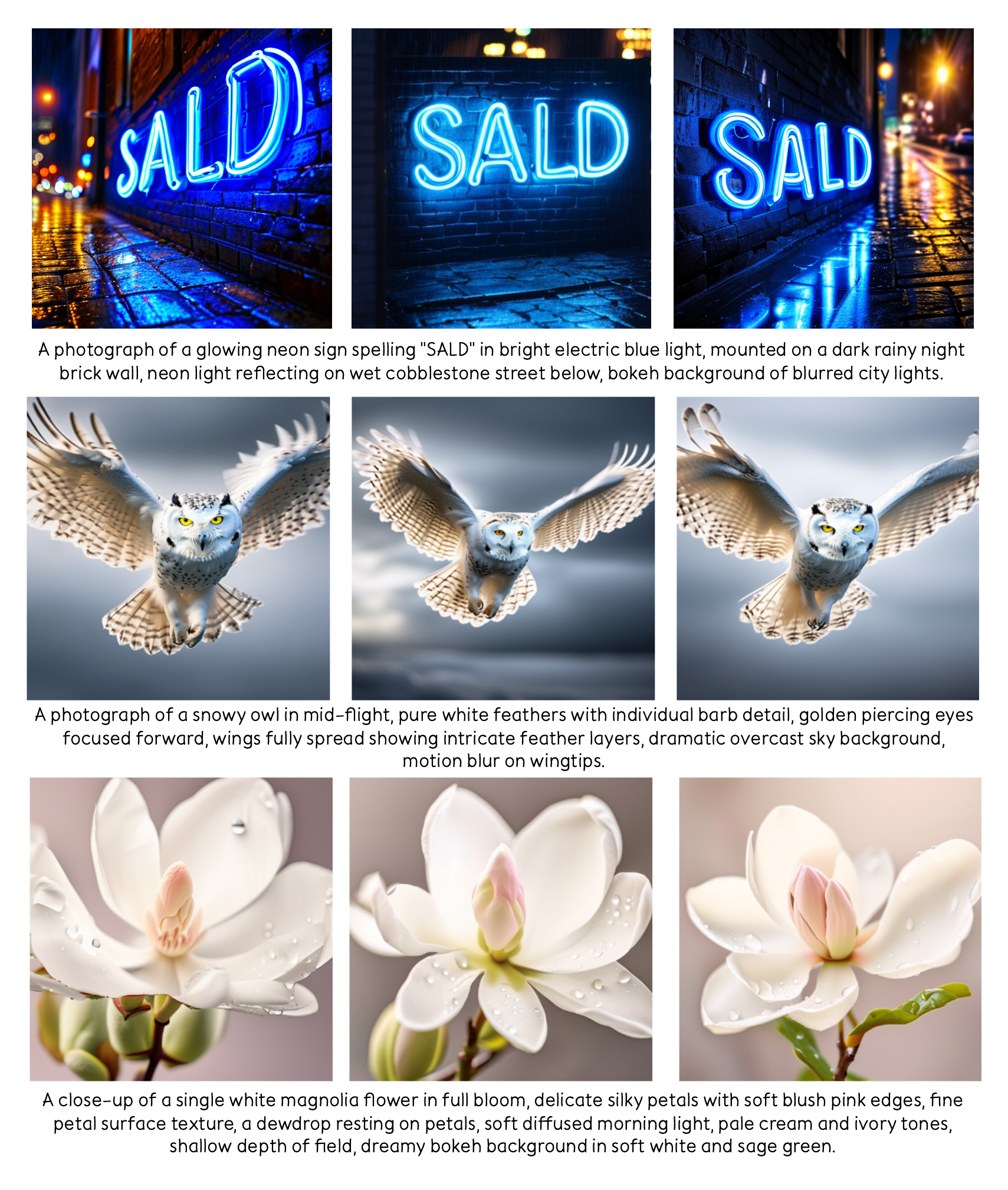}
    \caption{Qualitative samples generated by SALD across multiple datasets and architectures. Row~1: SANA~1.6B at $1024\!\times\!1024$ resolution trained on DALL$\cdot$E~3~1M. Row~2: SANA~1.6B generations on CUB-200. Row~3: PixArt-$\Sigma$ generations on Oxford-102 Flowers.}
  \label{fig:q2}
\end{figure}
%%%%%%%%%%%%%%%%%%%%%%%%%%%%%%%

% \clearpage
% \newpage
% \input{checklist.tex}

\end{document}